%% file: example_paper.tex
\documentclass{article}
\pdfoutput=1 

\usepackage{pifont}
\newcommand{\cmark}{\ding{51}}%
\newcommand{\xmark}{\ding{55}}%
\usepackage{microtype}
\usepackage{booktabs} % for professional tables
 
\usepackage[accepted]{icml2024}

\usepackage{amsmath}
\usepackage{amssymb}
\usepackage{mathtools}
\usepackage{amsthm}

\usepackage[english]{babel}
\usepackage[utf8]{inputenc}
\usepackage{amsmath,amssymb,amsthm}
\usepackage{bm}
\usepackage{parskip}
\usepackage[pdftex]{graphicx}
\usepackage{mathtools}
\usepackage[makeroom]{cancel}
\usepackage{mathtools}
\usepackage{tabularx}
\usepackage{svg}
\usepackage{adjustbox}
\usepackage{makecell}
\usepackage{tabularray}
\usepackage{color, colortbl}
\usepackage{algorithm2e}
\usepackage{multirow}
\usepackage[font={small}]{caption} 
\usepackage{subcaption}
\usepackage{hyperref}
\usepackage[capitalize,noabbrev]{cleveref}
\usepackage{xcolor}
\usepackage[T1]{fontenc}
\usepackage[disable]{todonotes}
\usepackage[rightcaption]{sidecap}
\usepackage{float}
\usepackage{tabstackengine}
\usepackage{hhline}
\usepackage{cellspace}
\definecolor{lightkhaki}{rgb}{0.6, 0.81, 0.93}

\definecolor{LightCyan}{rgb}{0.94, 0.9, 0.55}

\definecolor{amaranth}{rgb}{0.9, 0.4, 0.31}

\usepackage{tikz}
\newcommand*\circled[1]{\tikz[baseline=(char.base)]{
            \node[shape=circle,draw,inner sep=.7pt] (char) {#1};}}
\sidecaptionvpos{figure}{t}

\makeatletter
\let\original@footnote\footnote
\newcommand{\align@footnote}[1]{%
  \ifmeasuring@
    \chardef\@tempfn=\value{footnote}%
    \footnotemark
    \setcounter{footnote}{\@tempfn}%
  \else
    \iffirstchoice@
      \original@footnote{#1}%
    \fi
  \fi}
\pretocmd{\start@align}{\let\footnote\align@footnote}{}{}
\makeatother

\newcommand{\ben}[1]{\textcolor{blue}{#1}}
\newcommand{\benswitch}[1]{%
\ifthenelse{\equal{#1}{0}}{\renewcommand{\ben}[1]{}}{}}
\benswitch{0} 

\newcommand{\rebuttal}[1]{#1}
\newcommand{\rebuttalswitch}[1]{%
\ifthenelse{\equal{#1}{0}}{\renewcommand{\ben}[1]{}}{}}
\benswitch{1} 

\newtheorem{proposition}{Proposition}
\theoremstyle{definition}
\newtheorem{definition}{Definition}[section]

\sidecaptionvpos{figure}{t} 

\input{math_commands.tex}

\usepackage{url}

\icmltitlerunning{A Rate-Distortion View of Uncertainty Quantification}
\newcommand{\quotes}[1]{``#1''}

\begin{document}

\twocolumn[
\icmltitle{A Rate-Distortion View of Uncertainty Quantification}



\icmlsetsymbol{equal}{*}

\begin{icmlauthorlist}
\icmlauthor{Ifigeneia Apostolopoulou}{yyy}
\icmlauthor{Benjamin Eysenbach}{comp}
\icmlauthor{Frank Nielsen}{sch}
\icmlauthor{Artur Dubrawski}{yyy}
\end{icmlauthorlist}

\icmlaffiliation{yyy}{Machine Learning Department, AutonLab, Carnegie Mellon University}
\icmlaffiliation{comp}{Computer Science Department, Princeton University}
\icmlaffiliation{sch}{Sony Computer Science Laboratories Inc., Tokyo, Japan}

\icmlcorrespondingauthor{Ifigeneia Apostolopoulou}{iapostol@andrew.cmu.edu, ifiaposto@gmail.com}

\icmlkeywords{Machine Learning, ICML}

\vskip 0.3in
]

\let\orghat\hat
\newcommand{\that}[1]{\orghat{\kern-1.0pt #1}}
% \def\hat#1{\orghat{\kern0pt #1}}
% \newcommand*{\hat2}{\orghat{\kern0pt #1}}
% \newcommand{\fix}{\marginpar{FIX}}
% \newcommand{\new}{\marginpar{NEW}}

% \usepackage{letltxmacro}
% \LetLtxMacro{\originaleqref}{\eqref}
% \renewcommand{\eqref}{Eq.~\originaleqref}
\renewcommand*{\Eqref}[1]{Eq.~\ref{#1}}
\renewcommand*{\Figref}[1]{Fig.~\ref{#1}}
\def\train{{\mathrm{train}}}

\def\KL{{\mathrm{KL}}}

\newcommand*\rot{\rotatebox{90}}



\printAffiliationsAndNotice{}  % leave blank if no need to mention equal contribution

\input{abstract}
%%%%%%%%%%%%%%%%%%%%%%%%%%%%%%%
\input{intro}
%%%%%%%%%%%%%%%%%%%%%%%%%%%%%%%
\input{background}
%%%%%%%%%%%%%%%%%%%%%%%%%%%%%%%
\input{motivation}

%%%%%%%%%%%%%%%%%%%%%%%%%%%%%%%
\input{model}
%%%%%%%%%%%%%%%%%%%%%%%%%%%%%%%
\input{experiments}
\input{future_research}
\input{conclusion}

%\clearpage
\section*{Impact Statement}

This paper presents work whose goal is to advance the field of Machine Learning. There are many potential societal consequences of our work, none which we feel must be specifically highlighted here.

\section*{Acknowledgements}
This work was supported in part by DARPA award FA8750-17-2-0130, NSF grant 2038612, and by the U.S. Army Research Office and the U.S. Army Futures Command under contract W519TC-23-F-0045. The authors would like to thank Christos Faloutsos for insightful discussions on clustering methods that helped inform this work and Barnabás Póczos for his astute feedback on the manuscript draft.

\section*{Code Availability}
Publicly available code for reproducing the experiments can be found at:

\url{https://github.com/ifiaposto/Distance_Aware_Bottleneck}

%\clearpage
\bibliography{example_paper}{plain}
\bibliographystyle{icml2024}

%%%%%%%%%%%%%%%%%%%%%%%%%%%%%%%%%%%%%%%%%%%%%%%%%%%%%%%%%%%%%%%%%%%%%%%%%%%%%%%
%%%%%%%%%%%%%%%%%%%%%%%%%%%%%%%%%%%%%%%%%%%%%%%%%%%%%%%%%%%%%%%%%%%%%%%%%%%%%%%
% APPENDIX
%%%%%%%%%%%%%%%%%%%%%%%%%%%%%%%%%%%%%%%%%%%%%%%%%%%%%%%%%%%%%%%%%%%%%%%%%%%%%%%
%%%%%%%%%%%%%%%%%%%%%%%%%%%%%%%%%%%%%%%%%%%%%%%%%%%%%%%%%%%%%%%%%%%%%%%%%%%%%%%
\clearpage
\newpage
\appendix
\onecolumn

\input{appendix}

%%%%%%%%%%%%%%%%%%%%%%%%%%%%%%%%%%%%%%%%%%%%%%%%%%%%%%%%%%%%%%%%%%%%%%%%%%%%%%%
%%%%%%%%%%%%%%%%%%%%%%%%%%%%%%%%%%%%%%%%%%%%%%%%%%%%%%%%%%%%%%%%%%%%%%%%%%%%%%%

\end{document}

%% file: math_commands.tex
\usepackage{amsmath,amsfonts,bm}

\def\Figref#1{Figure~\ref{#1}}

\def\eqref#1{equation~\ref{#1}}
\def\Eqref#1{Equation~\ref{#1}}
\def\1{\bm{1}}
\newcommand{\train}{\mathcal{D}}

\DeclareMathAlphabet{\mathsfit}{\encodingdefault}{\sfdefault}{m}{sl}
\SetMathAlphabet{\mathsfit}{bold}{\encodingdefault}{\sfdefault}{bx}{n}

\newcommand{\KL}{D_{\mathrm{KL}}}

\DeclareMathOperator*{\argmax}{arg\,max}
\DeclareMathOperator*{\argmin}{arg\,min}

%% file: abstract.tex
%auto-ignore
\begin{abstract}
In supervised learning, understanding an input’s proximity to the training data can help a model decide whether it has sufficient evidence for reaching a reliable prediction. While powerful probabilistic models such as Gaussian Processes naturally have this property, deep neural networks often lack it. In this paper, we introduce Distance Aware Bottleneck (DAB), i.e., a new method for enriching deep neural networks with this property. Building on prior information bottleneck approaches, our method learns a codebook that stores a compressed representation of all inputs seen during training. The distance of a new example from this codebook can serve as an uncertainty estimate for the example. The resulting model is simple to train and provides deterministic uncertainty estimates by a single forward pass. Finally, our method achieves better out-of-distribution (OOD) detection and misclassification prediction than prior methods, including expensive ensemble methods, deep kernel Gaussian Processes, and approaches based on the standard information bottleneck.

\end{abstract}

%% file: intro.tex
%auto-ignore
\section{Introduction}
%%%%%%%%%%%%%
\begin{figure}[!ht]
\centering
\includegraphics[width=0.92\columnwidth, angle=0]%{icml2024/figures/distance_awareness_3.pdf}
{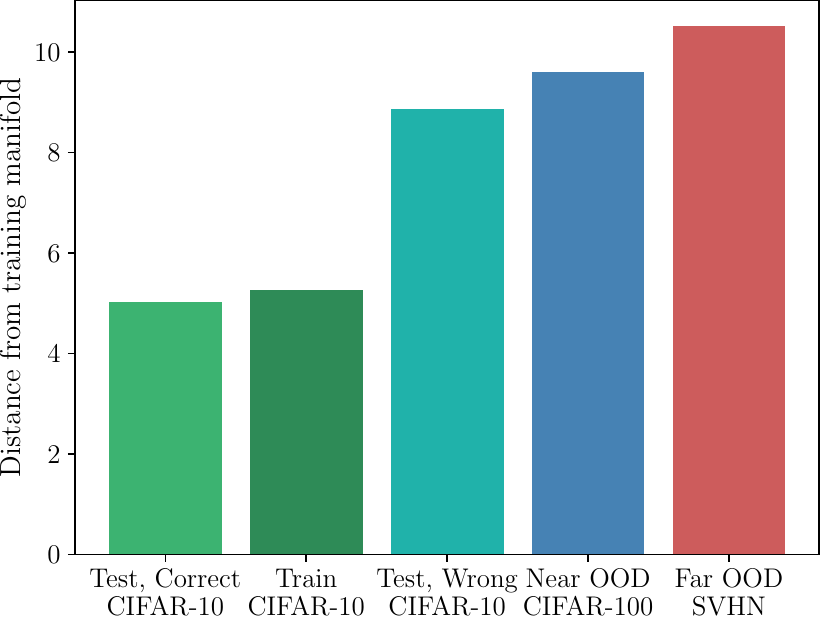}
\caption{\textbf{Distance awareness for principled uncertainty quantification.} A distance-aware model can measure the distance between input examples and the training examples. Our method learns distances where misclassified datapoints, semantic (near OOD), and domain (far OOD) deviations can be identified by larger distances. Our method learns and uses a codebook for representing the training dataset.
% This codebook effectively captures how different new datapoints are from the examples we used to train it.
Here, we report distances from a codebook trained on CIFAR-10.}
\label{fig:dab_distance_awareness}
\end{figure}
%%%%%%%%%%%%%%%%%%%
 Deep learning models that \quotes{know what they know} are becoming increasingly useful since they can better understand when to make confident predictions and when to ask for human help~\citep{50669}. Early approaches to uncertainty estimation built from probabilistic approaches tailored to deep neural networks (DNNs)~\citep{blundell2015weight, osawa2019practical, gal2016dropout} or deep ensembles~\citep{lakshminarayanan2017simple, wilson2020bayesian}. A shared characteristic of these methods is that they require multiple model samples to produce a reliable uncertainty estimate. Despite growing research interest in uncertainty quantification, we still lack reliable and efficient methods for real-world ML deployment.
 
%%%%%%%%%%%%%%%%%%%%%%%%
% As deep learning models move out of the lab into the real world and start impacting human lives~\citep{larson2019evaluation,gupta2021deep,jumper2021highly,jimenez2020drug,roy2022does,band2022benchmarking}, uncertainty quantification is becoming an increasingly important problem of machine learning. This is more evident in high-stakes and safety-critical decision-making scenarios where mistakes can have severe impacts such as medical diagnosis~\citep{roy2022does,band2022benchmarking} or autonomous vehicles~\citep{gupta2021deep}. In these cases, model uncertainty can help users decide when to trust the model or to defer to a human expert~\citep{50669}. Therefore, for the safe and responsible deployment of AI in the real world, the development of efficient uncertainty estimation for deep neural networks (DNNs) while maintaining their practical utility (accuracy and scalability) is crucial.
%%%%%%%%%%%%%%%%%%%%%%%%%%%
% Several methods have been proposed to endow DNNs with an estimate of uncertainty. These methods include probabilistic approaches tailored to DNNs~\citep{blundell2015weight, osawa2019practical, gal2016dropout} or ensemble approaches~\citep{lakshminarayanan2017simple, wilson2020bayesian}. A shared characteristic of these methods is that they require multiple model samples to produce a reliable uncertainty estimate.
 A recently emerged class of scalable uncertainty estimation methods, the \textit{Deterministic Uncertainty Methods} (DUMs)~\citep{pmlr-v162-postels22a,charpentier2023training}, affords uncertainty estimates with a single forward-pass. These methods are \textit{distance-aware}~\citep{liu2020simple} since they can quantify a distance score or measure of a new test example from previously trained-upon datapoints. Distance awareness renders DUMs a principled and theoretically motivated~\citep{liu2020simple,JMLR:v24:22-0479} solution to uncertainty quantification. \rebuttal{In particular, the distance score can indicate Out-Of-Distribution (OOD) examples of varying dissimilarity from the training datapoints or in-distribution areas where the model fails to generalize (\Figref{fig:dab_distance_awareness}). 
 
 Existing DUMs are usually tied to specific regularization techniques~\citep{miyato2018spectral, gulrajani2017improved} to mitigate feature collapse. Although such additional weight constraints help these methods reach state-of-the-art OOD detection results, they may undermine their calibration, i.e., how well a DNN can predict its incorrectness~\citep{pmlr-v162-postels22a}. Moreover, in the absence of similar constraints in large, pre-trained models, integration of current DUMs into industrial applications becomes difficult.}

In this work, we seek to improve the quality of uncertainty estimates using a single-model, deterministic characterization. The key contributions of this paper are as follows:
%%%%%%%%%%%%%
\begin{itemize}
    \item We formulate uncertainty quantification as the computation of a rate-distortion function to obtain a compressed representation of the training dataset. This representation is a set of prototypes defined as centroids of the training datapoints with respect to a distance measure. The expected distance of a datapoint from the centroids provides model's uncertainty for the datapoint (\Figref{fig:dab_distance_awareness}). 
    
    \item  We take a ``meta-probabilistic'' perspective to the rate-distortion problem. In particular, the distortion function operates on distributions of embeddings and corresponds to a statistical distance (\Figref{fig:overview}). To do so, we use the Information Bottleneck (IB) framework. The proposed formulation, the Distance Aware Bottleneck (DAB), jointly regularizes DNN’s representations and renders it distance-aware. 
    
    \item We design and qualitatively verify a practical deep learning algorithm that is based on successive estimates of the rate-distortion function to identify the centroids of the training data (Algorithm~\ref{alg:training_ig_vib}).  
    
    \item We show experimentally that our method can detect both OOD samples and misclassified samples. In particular, DAB  outperforms baselines when used for OOD tasks and closes the gap between single forward pass methods and expensive ensembles in terms of calibration (Tables~\ref{tab:ood_baselines},~\ref{table:calibration_score_}). 
\item  Finally, we show that DAB can be trained and applied post-hoc to a large, pre-trained feature extractor offering similar advantages for challenging and large-scale datasets (Table~\ref{table:imagenet}). 
\end{itemize}
%%%%%%%%%%%%%%%%%%%%%%
%%%%%%%%%%%%%%%%%%%%
\section{Related Work}\label{sec:related_work}

In this section, we provide an overview of existing DUMs and relate them to the proposed model. Most competitive DUMs can be taxonomized as Gaussian Process models or cluster-based approaches. 

Gaussian Processes (GPs) are intrinsically distance-aware models since they are defined by a kernel function that quantifies similarity to the training datapoints. SNGP~\citep{liu2020simple,JMLR:v24:22-0479} relies
on a Laplace approximation of the GP based on Random
Fourier Features (RFF)~\citep{rahimi2007random}. DUE~\citep{van2021feature} uses the inducing point GP approximation~\citep{titsias2009variational}. In Table~\ref{tab:gp_vs_daib}, we provide some analogies between Gaussian Processes and the model proposed in this work. 
%%%%%%%%%%%%%%%%%%%%%%
\begin{table}[!htbp]
    \caption{\textbf{\rebuttal{Analogies between GP and DAB.}}}
    \label{tab:gp_vs_daib}
\adjustbox{width=0.99\columnwidth,center}{
\centering
{\Huge
    \begin{tabular}{ c c  c }
        \toprule
\textbf{}     &  \textbf{ Gaussian Process}   
& \textbf{ Distance Aware Bottleneck} \\\midrule
\thead{ \Huge Compression of the \\  \Huge training dataset $\mathcal{D}_\train$} &  Inducing Points    
&  Codebook \\\hline
 Feature space &  \Huge $\mathbb{R}^d$        
&  \thead{ \Huge parameter space $\mathbf{\Theta}$ of a family of \\  \Huge distributions $\mathcal{P}=\{p(\boldsymbol{z};\boldsymbol{\theta})\mid \boldsymbol{\theta} \in \boldsymbol{\Theta} \}$} \\\hline
 Distance measure &  Euclidean norm        
& Statistical distance  \\\hline
\bottomrule
\vspace{-1in}
\end{tabular}}}
\end{table}
%%%%%%%%%%%%%%%%%%%%%%
Both SNGP and DUE enforce bi-Lipschitz constraints on the network by spectral normalization~\citep{miyato2018spectral} to encourage sensitivity and smoothness of the extracted features. 

In contrast, our work builds on IB methods ~\citep{alemi2018uncertainty} to avoid feature collapse. IB methods regularize the network by encouraging it to learn informative representations. Therefore, they are simple to implement and train. However, prior IB methods~\citep{alemi2018uncertainty} cannot sufficiently represent large and complex datasets. In this paper, we revise and augment prior IBs with a codebook capable of coding high-dimensional and multi-modal training distributions. The training is facilitated by a learning algorithm (Section~\ref{sec:learning_algorithm}) which, along with the gradient updates, matches the training examples with the entries of the codebook.
%In contrast, our model does not rely on additional feature regularization. As we will see later on (Section~\ref{sec:calibration_experiments}, Appendix \ref{sec:imagenet_setup}), it can be applied post-hoc to a large, pre-trained backbone network without interfering with its expensive training process. The ease of integrating uncertainty methods, such as the one proposed in this work, with pre-trained models makes it a compelling choice for industrial ML deployments such as Large Language Models (LLMs).

More closely related to our work is DUQ~\citep{van2020uncertainty}. Similar to our work, DUQ quantifies uncertainty as the distance from centroids responsible for representing the training data. The distance is computed in terms of a Radial Basis Function (RBF) kernel.  In contrast to our work, DUQ is trained to minimize a binary cross entropy loss function. This function assigns datapoints to clusters in a supervised manner. Therefore, the number of centroids is hardwired to the number of classes. This restriction renders the deployment of the model to regression tasks or classification tasks with a large number of classes difficult. On the other hand, our model provides a unified notion of uncertainty for both regression and classification tasks. Estimating regression uncertainty is important in many machine learning subfields. For example, in deep reinforcement learning uncertainty over the Q-values can be leveraged for efficient exploration or risk estimation~\citep{osband2016deep,lee2021sunrise, fujimoto2019off,wu2021uncertainty}. Effective DUMs, such as our model, could mend the current lack of both efficient and reliable uncertainty methods in unsupervised learning settings.

\rebuttal{Under a broad definition, data augmentation methods ~\citep{hendrycks2020augmix, pinto2022using} can also be considered DUMs. They improve network's learned representations by encouraging the model to be sensitive to or invariant against image perturbations. Design of such perturbations, however, requires domain expertise and/or prior knowledge. This requirement makes it difficult to extend existing data augmentation methods to other tasks (such as regression) or modalities (such as text). Here, we focus on principled, distance-aware DUMs, and borrowing terminology of~\citet{pmlr-v162-postels22a,Mukhoti_2023_CVPR}, unless otherwise noted, we use DUMs to refer to distance-aware DUMs.} Finally, we note that deep ensembles are included as a benchmark in our experiments, as they represent the current state-of-the-art for uncertainty quantification. However, while simple in concept and implementation, their computational and memory cost are prohibitive.

%% file: background.tex
%auto-ignore
\begin{figure*}[!ht]
  %\centering
\begin{subfigure}[b]{0.245\textwidth}
\centering
\includegraphics[width=0.99\textwidth]{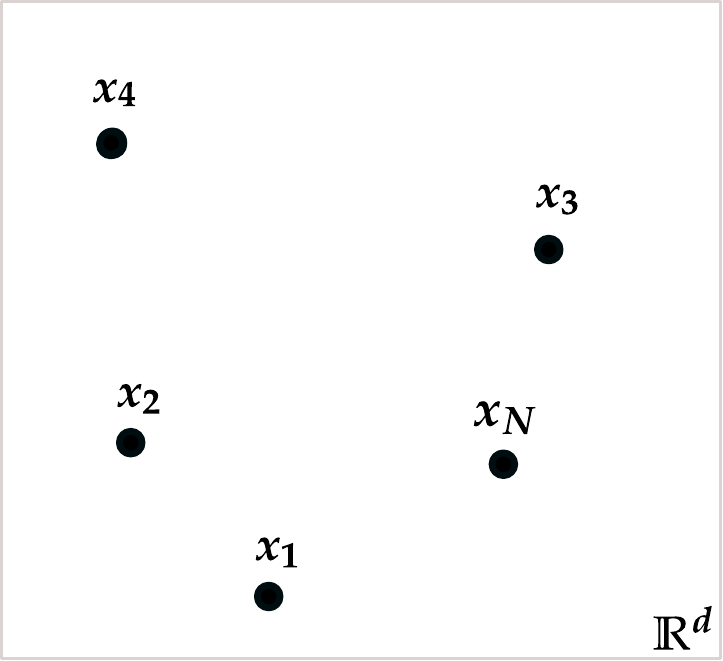}
\caption{$\mathcal{D}_\train$ as a set of points in the Euclidean space $\mathbb{R}^d$.}
\label{d_train}
\end{subfigure}
  \begin{subfigure}[b]{0.245\textwidth}
\centering\includegraphics[width=0.99\textwidth]{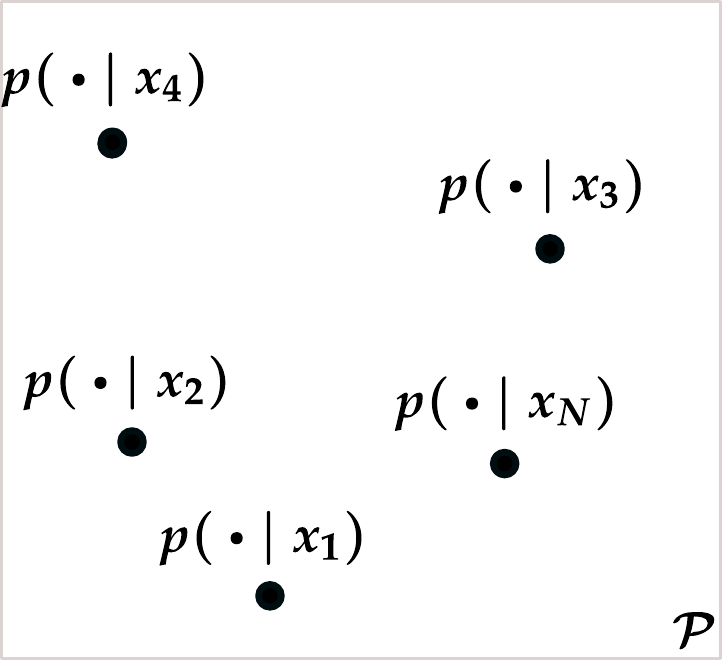}
\caption{$\mathcal{D}_\train$ as a set of points in distribution space $\mathcal{P}$.}
 \label{d_train_enc}
  \end{subfigure}
\begin{subfigure}[b]{0.245\textwidth}
\centering\includegraphics[width=0.99\textwidth]{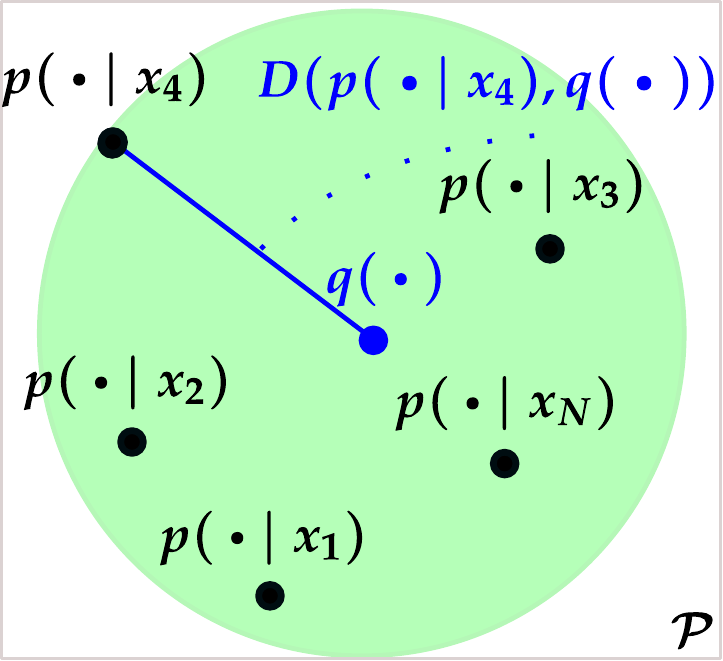}
\caption{Support of $\mathcal{D}_\train$ as a \\ statistical ball ($k=1$).}
  \label{d_train_enc_center}
  \end{subfigure}
\begin{subfigure}[b]{0.245\textwidth}
\centering\includegraphics[width=0.99\textwidth]{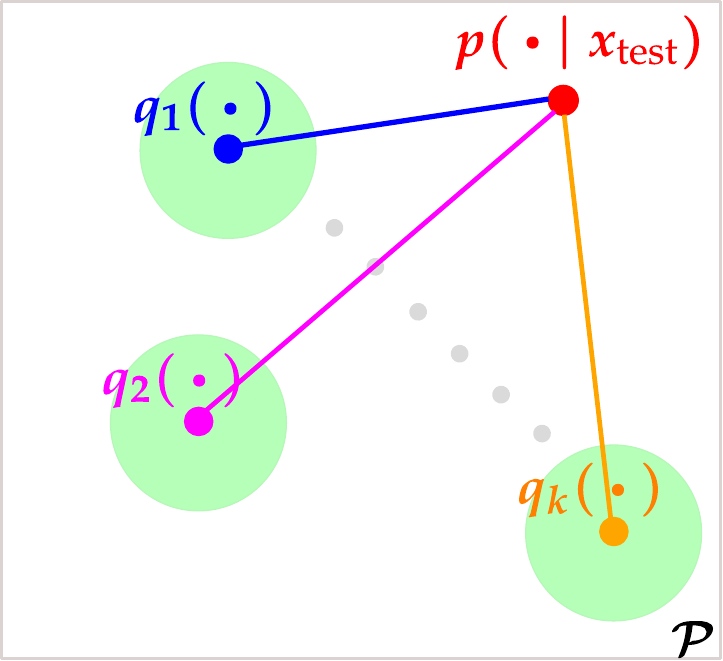}
\caption{Distance from the codebook of encoders ($k>1$).}
  \label{d_train_enc_center_codebook}
  \end{subfigure}
\caption{\textbf{Overview of DAB.} Uncertainty quantification in DAB is based on compressing the training dataset $\mathcal{D}_\train$ by learning a codebook and computing distances from the codebook. The datapoints in $\mathcal{D}_\train$, originally lying in $\mathbb{R}^d$ (\ref{d_train}), are embedded into distribution space $\mathcal{P}$ of a parametric family of distributions through their encoders (\ref{d_train_enc}). Compression of $\mathcal{D}_\train$ amounts to finding the {\color{blue}centroids} of the encoders in terms of a statistical distance $D$ (\ref{d_train_enc_center}). For complex datasets, usually multiple centroids are needed (\ref{d_train_enc_center_codebook}). The uncertainty for a previously unseen {\color{red} test datapoint} is quantified by its expected distance from the codebook: $\mathrm{uncertainty}(\color{red}{x_{\text{test}}}\color{black})=\mathbb{E}[D(\color{red} p(\boldsymbol{z} \mid \boldsymbol{x}_{\text{test}};\boldsymbol{\theta})\color{black},\color{blue} q_\kappa (\boldsymbol{z}; \boldsymbol{\phi}) \color{black})]$.}
\label{fig:overview}
%\vspace{-0.01in}
\end{figure*}
\section{Preliminaries}
%%%%%%%%%%%%%
\subsection{Information Bottleneck}
%%%%%%%%%%%%%
The Information Bottleneck (IB)~\citep{tishby2000information} provides an information-theoretic view for balancing the complexity of a stochastic encoder $Z$ for input $X$ \footnote{We denote random variables as $X, Y, Z$ and their instances as $\boldsymbol{x}, \boldsymbol{y}, \boldsymbol{z}$.} and its predictive capacity for the desired output $Y$. The IB objective is:
\begin{align}
\min_{\boldsymbol{\theta}} \quad -I_{}(Z,Y; \boldsymbol{\theta}) +\beta\,I_{}(Z,X;\boldsymbol{\theta}),
\label{eq:ib_ii}
\end{align}
where $\beta \ge 0$ is the trade-off factor between the accuracy term $I_{}(Z,Y;\boldsymbol{\theta})$ and the complexity term $I_{}(Z,X;\boldsymbol{\theta})$. $\boldsymbol{\theta}$ denotes the parameters of the distributional family of encoder $p_{}(\boldsymbol{z}\mid \boldsymbol{x}; \boldsymbol{\theta})$\footnote{$\boldsymbol{\theta}$ will represent a function implemented by a neural network. For input $\boldsymbol{x}$ it computes the parameters of the conditional distribution $p(\cdot \mid \boldsymbol{x};{\boldsymbol{\theta}})$ in its output. For example, for a Gaussian with diagonal covariance: $\boldsymbol{\theta}({\boldsymbol{x}})=\{\mu(\boldsymbol{x}), \sigma(\boldsymbol{x})\}, \mu(\boldsymbol{x})\in \mathbb{R}^d, \sigma(\boldsymbol{x}) \in \mathbb{R}^d_{\ge0}$. Optimization with respect to $\boldsymbol{\theta}$ will refer to optimization with respect to the weights of network $\boldsymbol{\theta}$.}.
In words, training by~\Eqref{eq:ib_ii} encourages the model to find a representation $Z$ that is maximally expressive about output $Y$ while being maximally compressive about input $X$. 
Typically, the mutual information terms in~\Eqref{eq:ib_ii} cannot be computed in closed-form since % they are defined over decoder $p(\boldsymbol{y} \mid \boldsymbol{z})$ and marginal $p(\boldsymbol{z})$, which, in turn, 
they involve intractable marginal distributions (\Eqref{eq:exact_decoder},~\ref{eq:exact_marginal}). The Variational Information Bottleneck (VIB)~\citep{alemi2016deep} considers parametric approximations $m(\boldsymbol{y}\mid\boldsymbol{z};\boldsymbol{\theta})$, $q(\boldsymbol{z}; \boldsymbol{\phi})$ to these marginals belonging to a distributional family parametrized by $\boldsymbol{\theta}$ \footnote{In the rest of the paper, we use $\boldsymbol{\theta}$ to denote the joint set of parameters of encoder and variational decoder.} and $\boldsymbol{\phi}$ respectively. The VIB objective (\Eqref{eq:vib_loss}) maximizes a lower bound of $I(Z,Y; \boldsymbol{\theta})$ and minimizes an upper bound of $I(Z,X;\boldsymbol{\theta})$ . In this work, we reconsider the complexity term. The upper bound of this term is an expected Kullback-Leibler divergence:
%%%%%%%%%%%%%
\begin{align}
I(Z,X; \boldsymbol{\theta}) &= \mathbb{E}_{X}\big[{D}_\KL(p( \boldsymbol{z}\mid \boldsymbol{x};\boldsymbol{\theta}), p(\boldsymbol{z}))] \nonumber  \\ &\le \mathbb{E}_{X}\big[{D}_\KL(p(\boldsymbol{z}\mid \boldsymbol{x};\boldsymbol{\theta}), q(\boldsymbol{z};\boldsymbol{\phi}))\big]\label{eq:tau_z_centroid}.
\end{align}
The expectation in~\Eqref{eq:tau_z_centroid} is taken, in practice, with respect to the empirical distribution of the training dataset $\mathcal{D}_{\train}=\{(\mathbf{x}_i,\mathbf{y}_i)\}_{i=1}^{N}$:
\begin{align}
I(Z,X ;\boldsymbol{\theta})  \lessapprox \frac{1}{N} \sum_{i=1}^{N}{D}_\KL(p(\boldsymbol{z}\mid\boldsymbol{x}_i;\boldsymbol{\theta}), q(\boldsymbol{z};\boldsymbol{\phi})).
\label{eq:vib_i_x_z_e}
\end{align}
%%%%%%%%%%%%%
\subsection{Rate Distortion Theory}
The rate-distortion theory~\citep{berger1971, berger1998lossy,cover1999elements} quantifies the fundamental limit of data
compression, i.e., at least how many bits are needed to quantize data coming from a stochastic source given a desired fidelity. Formally, consider random variable $X \sim p(\boldsymbol{x})$ with support set\footnote{Support set of $X \sim p(\boldsymbol{x})$ is the set $\mathcal{X}=\{\boldsymbol{x}: p(\boldsymbol{x}) >0 \}$.} $\mathcal{X}$. Data coming from source $X$ will be compressed by mapping them to a random variable $\hat{X}$ with support set $\mathcal{\hat{X}}$. It is common to refer to $\hat{X}$ as the source code or quantization of $X$. In this work, we consider a discrete source over $\mathcal{D}_\train$ following the empirical distribution. The formal description is deferred to Section~\ref{sec:ua_ib}.

The quality of the reconstructed data is assessed using a distortion function $D:\mathcal{X} \times \mathcal{\hat{X}} \rightarrow \mathbb{R}^{+}$. 
The rate-distortion function is the minimum achievable rate (number of bits) of the quantization scheme for a prescribed level of expected distortion. In Lagrange formulation, it is the problem: 
\begin{align}
& R \triangleq \min_{p(\hat{\boldsymbol{x}}\mid \boldsymbol{x})} \quad  I (X;\hat{X}) + \alpha \mathbb{E}_{X,\hat{X}}[D(\boldsymbol{x},\hat{\boldsymbol{x}})],\label{eq:r_d}
\end{align}
where $\alpha  \ge 0$ is the optimal Lagrange multiplier that corresponds to a distortion constraint $\mathbb{E}_{X,\hat{X}}[D(\boldsymbol{x},\hat{\boldsymbol{x}})] \le d$.\footnote{Formally, $\alpha$ is a function of $d$: $\alpha\equiv\alpha(d)$. However, we omit this dependence for notational brevity.} It can be shown that the problem in~\Eqref{eq:r_d} is equivalent to a double minimization problem over $p(\hat{\boldsymbol{x}})$, $p(\hat{\boldsymbol{x}}\mid \boldsymbol{x})$ (Lemma 10.8.1 of~\citet{cover1999elements}). This equivalence enables an alternating minimization algorithm~\citep{csiszar1984information} -- the Blahut–Arimoto (BA) algorithm~\citep{blahut1972computation,matz2004information} -- for solving $R$. In practice, numerical computation
of the rate-distortion function through the BA algorithm is often infeasible,
primarily due to lack of knowledge of the optimal support of $\hat{X}$. The \textit{Rate Distortion Finite Cardinality (RDFC)} formulation~\citep{rose1994mapping, banerjee2004information} simplifies the computation of $R$ by assuming finite support $\hat{\mathcal{X}}$ that is jointly optimized:
\begin{align}
& \min_{\mathcal{\hat{X}}, p(\hat{\boldsymbol{x}}\mid \boldsymbol{x})} \quad I(X,\hat{X}) + \alpha\, \mathbb{E}_{X, \hat{X}}[ D(\boldsymbol{x},\hat{\boldsymbol{x}})] \nonumber \\
& \text{ subject to: } \mid \mathcal{\hat{X}} \mid=k.\label{eq:r_d_f_c_lag}
\end{align} 
The RDFC objective in~\Eqref{eq:r_d_f_c_lag} can be greedily estimated by alternating optimization over $\mathcal{\hat{X}}$, $p(\hat{\boldsymbol{x}})$, $p(\hat{\boldsymbol{x}}\mid \boldsymbol{x})$ yielding a solution that is locally optimal~\citep{banerjee2004information}.

%% file: motivation.tex
%auto-ignore
\section{Motivation}\label{sec:motivation}
%%%%%%%%%%%
The idea behind our approach is visualized in \Figref{fig:overview}. The crux of our approach is the observation that the variational marginal $q(\boldsymbol{z};\boldsymbol{\phi})$ in~\Eqref{eq:tau_z_centroid} and~\Eqref{eq:vib_i_x_z_e} encapsulates all encoders $p(\boldsymbol{z} \mid \boldsymbol{x}_i;\boldsymbol{\theta})$ of datapoints in $\mathcal{D}_{\train}$ encountered during training. To see this formally, we introduce a random variable $P_X$ defined by $X\sim p(\boldsymbol{x})$. The value of $P_X$ corresponding to $\boldsymbol{x}$ is the encoder's density $p(\boldsymbol{z}\mid \boldsymbol{x};\boldsymbol{\theta})$ (\Figref{d_train},~\ref{d_train_enc}). In other words, a value of $P_X$ is itself a probability distribution.
From proposition 1 of~\citet{banerjee2005clustering}, $ \mathbb{E}[P_X]$ is the unique centroid of encoders $p(\boldsymbol{z}\mid \boldsymbol{x}_i; \boldsymbol{\theta})$ with respect to \textit{any} Bregman divergence $D_{f}$ defined by a strictly convex and differentiable function $f$~\citep{bregman1967relaxation, brekelmans2022rhotau} (def.~\ref{def:bregman_divergence}):
%%%%%%%%%%%%%
\begin{align}
 \mathbb{E}[P_X] & = \frac{1}{N}\sum_{i=1}^{N}p(\boldsymbol{z} \mid \boldsymbol{x}_i; \boldsymbol{\theta}) \nonumber \\ & = \argmin_{q(\boldsymbol{z})} \frac{1}{N}\sum_{i=1}^{N} D_{f}(p(\boldsymbol{z} \mid \boldsymbol{x}_i; \boldsymbol{\theta}) , q(\boldsymbol{z})).
\label{eq:marginal_as_meanb} 
\end{align} 
%%%%%%%%%%%%% 
We note that the upper bound in~\Eqref{eq:vib_i_x_z_e} emerges as a special case of the minimization objective in~\Eqref{eq:marginal_as_meanb}. This is because the Kullback-Leibler divergence is a Bregman divergence~\citep{azoury2001relative,nielsen2007bregman} with the negative entropy as the generator function $f$~\citep{frigyik2008functional, csiszar1995generalized}\footnote{We consider functional Bregman divergences, i.e., the generalization of Bregman divergence that acts on functions.}. 
Therefore, $q(\boldsymbol{z}; \boldsymbol{\phi})$ in the VIB can also be viewed as a variational centroid of the training datapoints' encoders (\Figref{d_train_enc_center}). In this work, we consider learnable parameters $\boldsymbol{\phi}$. Under this view, the role of the regularization term $I(Z,X;\boldsymbol{\theta})$ when upper bounded by~\Eqref{eq:tau_z_centroid} is now twofold: i) it both regularizes encoder $p(\boldsymbol{z} \mid \boldsymbol{x}; \boldsymbol{\theta})$ and ii) it learns a distributional centroid $q(\boldsymbol{z}; \boldsymbol{\phi})$ for encoders $p(\boldsymbol{z} \mid \boldsymbol{x}_i; \boldsymbol{\theta})$ of training examples $\boldsymbol{x}_i$.

For complex data, it usually does not suffice to represent $\{p(\boldsymbol{z} \mid \boldsymbol{x}_i; \boldsymbol{\theta})\}_{i=1}^{N}$ by a single distribution  $q(\boldsymbol{z}; \boldsymbol{\phi})$.
Therefore, we will need to learn a collection (codebook) of $k$ centroids $\{q_{\kappa}(\boldsymbol{z}; \boldsymbol{\phi})\}_{\kappa=1}^{k}$\footnote{$\boldsymbol{\phi}$ will represent the joint set of parameters of all centroids $q_\kappa$. } (\Figref{d_train_enc_center_codebook}). In Section~\ref{sec:model}, we formalize how such a set of distributions can be learned and used to effectively quantify distance from $\mathcal{D}_\train$.

%% file: model.tex
%auto-ignore
\section{Distance Aware Bottleneck}\label{sec:model}
%%%%%%%%%%%
\subsection{Model}\label{sec:ua_ib}
In this section, we present the \textit{Distance Aware Bottleneck (DAB)}: An IB problem with a complexity constraint that regularizes the network \textit{and} renders the network distance-aware given a compressed representation of $\mathcal{D}_\train$. %\rebuttal{The model gauges its prediction's trustworthiness by evaluating how closely a new input aligns with this representation.} 
We keep an information-geometric interpretation of this representation. In this case, the features of $\boldsymbol{x}$ and the codes used for computing distance from $\mathcal{D}_\train$ lie in the parameter space of a distributional family $\mathcal{P}$\footnote{For example, for the family $\mathcal{P}=\{p(\boldsymbol{z};\boldsymbol{\theta})\mid \boldsymbol{\theta} \in \boldsymbol{\Theta} \}$ of $d-$ dimensional Normal distributions, the parameter space is $\boldsymbol{\Theta}=\{\mathbb{R}^d\times \mathrm{Sym}^+(d;\mathbb{R})\}$ with $\mathrm{Sym}^+(d;\mathbb{R})$ the set of $d\times d$ symmetric, positive definite matrices.} (\Figref{d_train_enc}). As we will see in Section~\ref{sec:uncertainty_dab}, the characterization of datapoints at a distributional granularity provides the model with deterministic uncertainty estimates. Moreover, we argue that an input $\boldsymbol{x}$ is better characterized by its encoder $p(\boldsymbol{z}\mid \boldsymbol{x};\boldsymbol{\theta})$. This is because standard Euclidean distances might disregard aspects of data that are essential for characterizing distance from $\mathcal{D}_\train$. In Section~\ref{sec:experiments}, we empirically confirm our hypothesis.  

The mathematical construction of our work was alluded in Section~\ref{sec:motivation} when we introduced random variable $P_X$. $P_X$ is defined by $X$ and takes as value the distribution  $p(\boldsymbol{z} \mid \boldsymbol{x};\boldsymbol{\theta})$, i.e., the encoder, as we sample $X\sim p(\boldsymbol{x})$. %support of $P_X$ is the set of encoders: $\mathcal{P}_X= \{p(\boldsymbol{z} \mid \boldsymbol{x})\}_{\boldsymbol{x}\in \mathcal{X}}$. 
In its empirical form over a finite number of $N$ training datapoints $\mathcal{D}_\train$, the distribution of  $P_X$ is a discrete \textit{distribution over distributions}: $P_X$ is discrete taking values in the set $\mathcal{P}_X=\{p(\boldsymbol{z} \mid \boldsymbol{x}_i;\boldsymbol{\theta})\}_{i=1}^{N}$ with probability $1/N$. We also define a random variable ${Q}$. By fixing the number $k$ of distributional centroids, ${Q}$ takes values $[q_1(\boldsymbol{z};\boldsymbol{\phi}), q_2(\boldsymbol{z};\boldsymbol{\phi}), \dots, q_k(\boldsymbol{z};\boldsymbol{\phi})]$ following distribution $\pi=[\pi(1), \pi(2), \dots, \pi(k)]$. We will refer to its support set $\mathcal{Q}=\{q_\kappa(\boldsymbol{z};\boldsymbol{\phi})\}_{\kappa=1}^{k}$ as the codebook. $\pi_{\boldsymbol{x}}$ is the conditional assignment probabilities of encoder $p(\boldsymbol{z} \mid \boldsymbol{x};\boldsymbol{\theta})$ to the centroids such that $\pi_{\boldsymbol{x}}=[\pi_{\boldsymbol{x}}(1), \pi_{\boldsymbol{x}}(2), \dots, \pi_{\boldsymbol{x}}(k)]$ with:
\begin{equation}
    \pi_{\boldsymbol{x}}(\kappa)=p({Q}=q_\kappa(\boldsymbol{z};\boldsymbol{\phi}) \mid P_X =p(\boldsymbol{z} \mid \boldsymbol{x};\boldsymbol{\theta})).
    \label{eq:conditional_pq}
\end{equation}
Compression of $\mathcal{D}_\train$ is phrased as a RDFC problem (\Eqref{eq:r_d_f_c_lag}) for the source of encoders $P_{X}$ using the source code $Q$: %The regularization term of IB is now defined as a fixed and finite cardinality $k$, rate-distortion function (equation \ref{eq:r_d_f_c}):
\begin{align}
 & R^k(\boldsymbol{\theta}) = \min_{\substack{{\mathcal{Q}}, \pi_{\boldsymbol{x}}}} \quad \mathcal{L}_{\mathrm{RDFC}} \quad \text{ subject to:} 
 \mid {\mathcal{Q}} \mid=k, \text{ where: }  \label{eq:r_d_f_c_dtrain}  \\
& \mathcal{L}_{\mathrm{RDFC}} \triangleq  
 \nonumber \\ 
& I(P_X,{Q};\boldsymbol{\theta},\boldsymbol{\phi}) +  \alpha \mathbb{E}_{P_X,{Q}}\big[ D(p(\boldsymbol{z} \mid \boldsymbol{x};\boldsymbol{\theta}),q_\kappa(\boldsymbol{z};\boldsymbol{\phi}))\big]. 
\label{eq:r_d_f_c_dtrain_loss}
\end{align}
At this point, we underline that the source of encoders $P_X$ depends on $\boldsymbol{\theta}$. Since centroids $q_\kappa(\boldsymbol{z};\boldsymbol{\phi})$ are used to quantize the set of encoders in $\mathcal{D}_\train$, we will also call them code distributions. Albeit in this work we investigate only the behavior of the Kullback-Leibler divergence, the distortion function $D$ in~\Eqref{eq:r_d_f_c_dtrain_loss} can be $\textit{any}$ statistical distance measure between two \textit{probability distributions}.

Optimizing with respect to the support set $\mathcal{Q}$ amounts to optimizing with respect to parameters $\boldsymbol{\phi}$ of codes $q_\kappa(\boldsymbol{z};\boldsymbol{\phi})$. Therefore, the problem in~\Eqref{eq:r_d_f_c_dtrain} can be written as:
\begin{align}
 R^k(\boldsymbol{\theta}) = \min_{\substack{\boldsymbol{\phi}, \pi_{\boldsymbol{x}}}} \quad \mathcal{L}_{\mathrm{RDFC}}, 
\label{eq:r_d_f_c_dtrain_ii}
\end{align}
where $\mathcal{L}_{\mathrm{RDFC}}$ is defined in~\Eqref{eq:r_d_f_c_dtrain_loss}.
% \begin{align}
% & R^k(D^{*};\boldsymbol{\theta}) = \min_{\substack{\boldsymbol{\phi}, \pi_{\boldsymbol{x}}}} \quad %\min_{\phantom{{Q}}\pi_{\boldsymbol{x}}} 
% I(P_X,{Q};\boldsymbol{\theta}, \boldsymbol{\phi}) \quad \text{ subject to:} \nonumber \\ & \mathbb{E}_{P_X,{Q}}[ D(p(\boldsymbol{z} \mid \boldsymbol{x};\boldsymbol{\theta}),q_\kappa(\boldsymbol{z};\boldsymbol{\phi}))] \le  D{^*}.
% \label{eq:r_d_f_c_dtrain_ii}
% \end{align}
%%%%%%%%%%%%%
%%%%%%%%%%%%%%%
DAB replaces the rate term $I(Z,X;\boldsymbol{\theta})$ of the IB (~\Eqref{eq:ib_ii}) with the achievable rate $R^k(\boldsymbol{\theta})$~(\Eqref{eq:r_d_f_c_dtrain_ii}). Formally, a DAB of cardinality $k$ is defined as:
\begin{align}
& \min_{\boldsymbol{\theta}} \quad -I(Z,Y;\boldsymbol{\theta}) +\beta R_{}^k(\boldsymbol{\theta})  \Longleftrightarrow  
\min_{\boldsymbol{\theta}} \min_{\boldsymbol{\phi},\pi_{\boldsymbol{x}}} \quad \mathcal{L}_{\mathrm{DAB}}, \nonumber \\  & \text{ where: }\mathcal{L}_{\mathrm{DAB}} 
\triangleq -I_{}(Z,Y;\boldsymbol{\theta}) + \beta I(P_X, Q;\boldsymbol{\theta},\boldsymbol{\phi}) + \nonumber \\ & \alpha \beta \ \mathbb{E}_{P_X,{Q}}\big[ D(p(\boldsymbol{z} \mid \boldsymbol{x};\boldsymbol{\theta}),q_\kappa(\boldsymbol{z};\boldsymbol{\phi}))\big].
\label{eq:ua_ib}
\end{align}
Training the network with the loss function $\mathcal{L}_\mathrm{DAB}$ encourages encoders $p(\boldsymbol{z} \mid \boldsymbol{x};\boldsymbol{\theta})$ whose samples $\boldsymbol{z}$ are informative about output $\boldsymbol{y}$ while staying statistically close to codes $q_\kappa(\boldsymbol{z};\boldsymbol{\phi})$. To get a better insight into~\Eqref{eq:ua_ib}, we consider two edge cases. In the case of a single code, i.e., $k=1$, with $D$ taken as the Kullback-Leibler divergence,~\Eqref{eq:ua_ib} is equivalent to the empirical form (\Eqref{eq:vib_i_x_z_e}) of the VIB~\citep{alemi2016deep} (\Eqref{eq:vib_loss}) with regularization coefficient $\alpha \times \beta$. For $k=N$, the optimal codes would correspond to training datapoints' encoders: $q_\kappa(\boldsymbol{z};\boldsymbol{\phi}) = p(\boldsymbol{z}\mid \boldsymbol{x}_\kappa;\boldsymbol{\theta})$ yielding zero compression (and regularization).

We note that DAB's objective (\Eqref{eq:ua_ib}) uses two separate terms for accuracy and for controlling the distance of training datapoints from the codebook. Such formulation 
enables DAB to choose between pulling correctly classified datapoints close to the codebook (being less uncertain) or pushing misclassified datapoints away (more uncertain), ultimately leading to better calibration (Section~\ref{sec:calibration_experiments}). %The same principle could enable DAB's Outlier Exposure (OE)~\citep{hendrycks2019oe} to improve its OOD capability by encouraging large distances for the OE datapoints. 

As in the standard VIB~\citep{alemi2016deep}, $I(Z,Y;\boldsymbol{\theta})$ can be estimated by the lower bound $\mathbb{E}_{X,Y,Z}[\log m(\boldsymbol{y}\mid \boldsymbol{z};\boldsymbol{\theta})]$ that is maximized with respect to variational decoder $m(\boldsymbol{y}\mid \boldsymbol{z};\boldsymbol{\theta})$ (\Eqref{eq:vib_loss}). We emphasize that, in this work, the decoder does not utilize model's proposed distance for eventually \textit{improving} its predictions. In DAB, this could be achieved by designing a \textit{stochastic} decoder that induces variance proportionate to the estimated distance (uncertainty) in its final prediction. Such a decoder could be viewed as a distance-aware epinet~\citep{osband2021epistemic} and its design is left as future work.
%%%%%%%
\begin{figure*}[!ht]
     \centering
     \begin{subfigure}[b]{0.49\textwidth}
         \centering         \includegraphics[width=0.9\linewidth]{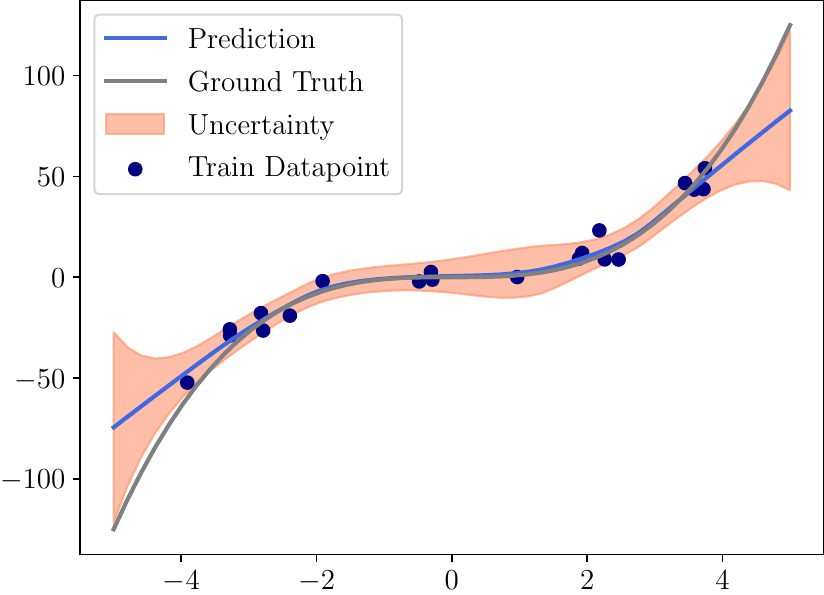}
         \caption{A single cluster of training data points.}
         \label{fig:regression_uniform}
     \end{subfigure}
     \hfill
     \begin{subfigure}[b]{0.49\textwidth}
         \centering         \includegraphics[width=0.9\linewidth]{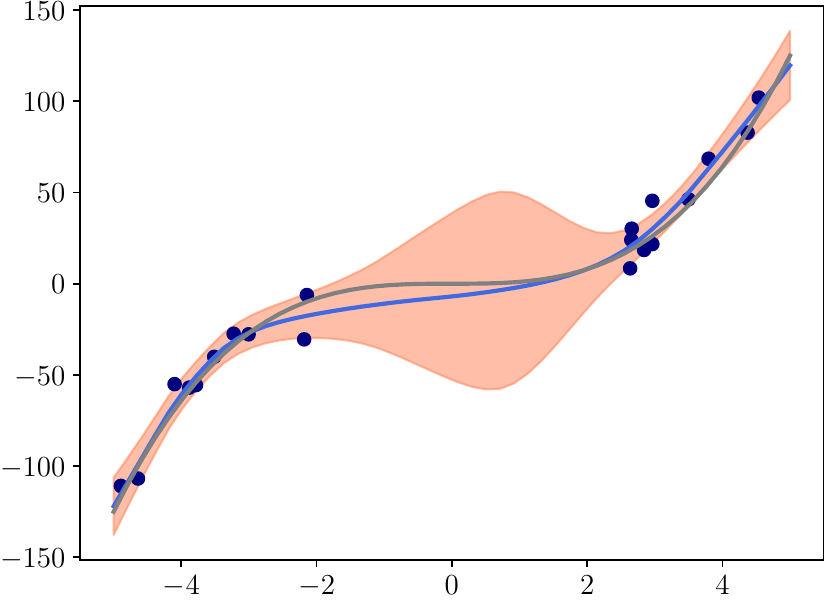}
         \caption{Two clusters of training data points.}
         \label{fig:regression_clusters}
     \end{subfigure}
\caption{\textbf{Uncertainty estimation on noisy regression tasks.} We consider the Kullback-Leibler divergence as the distortion function in the uncertainty score of~\Eqref{eq:distance_3}. A larger distance from the training datapoints (blue dots) is consistently quantified by higher uncertainty (width of pink area). Moreover, the true values lie well within $\pm 2 \times $ the proposed uncertainty score around the predictive mean. }
\label{fig:uncertainty_regression}
\end{figure*}
%%%%%%%%%%%%%
\subsection{Learning Algorithm}\label{sec:learning_algorithm}
%%%%%%%%%%%%%
The optimization problem of~\Eqref{eq:ua_ib} can be solved by alternating minimizations~\citep{banerjee2004information}. We note that $I(P_X,{Q};\boldsymbol{\theta},\boldsymbol{\phi})$ in~\Eqref{eq:ua_ib} is tractable since $P_X, {Q}$ are discrete random variables taking $N$ (size of training dataset) and $k$ (size of codebook) possible values, respectively. At each step, a single block of parameters is optimized. The most recent value is used for the parameters that are not optimized at the step. The internal minimization step corresponds to the computation of a RDFC (\Eqref{eq:r_d_f_c_lag}). The minimization steps are summarized as follows:
\[\rotatebox{90}{\text{\hspace{-0.18in}repeat}}\begin{cases}
t.&\text{Update decoder } m(\boldsymbol{y} \mid \boldsymbol{z};\boldsymbol{\theta}), \text{encoder } \\ & p(\boldsymbol{z} \mid \boldsymbol{x};\boldsymbol{\theta}):   \boldsymbol{\theta} \leftarrow \boldsymbol{\theta} - \eta_{\boldsymbol{\theta}} \nabla_{\boldsymbol{\theta}} \mathcal{L}_{\mathrm{DAB}}  \\
t+1.&\text{Update  }  \pi_{\boldsymbol{x}} \text{ from~\Eqref{eq:vib_ii_e_step_sol}}\\ 
t+2.&\text{Update centroids } q_\kappa(\boldsymbol{z};\boldsymbol{\phi}): \\  &\boldsymbol{\phi} \leftarrow \boldsymbol{\phi} - \eta_{\boldsymbol{\phi}} \nabla_{\boldsymbol{\phi}} \mathcal{L}_{\mathrm{DAB}}\\ 
t+3.&\text{Update } \pi \text{ 
 from~\Eqref{eq:prior_cluster_prob_update}}
  %\end{bmatrix}
\end{cases}\]%}
%%%%%%%%%%%%%
Steps $t+1$, $t+3$ are computationally cheap and can be performed analytically with a single forward pass:
%\begin{minipage}{.65\textwidth}
\begin{align}
\pi_{{\boldsymbol{x}}}(\kappa) &= \frac{\pi(\kappa)}{\mathcal{Z}_{\boldsymbol{x}}(\alpha)}\exp\big(-\alpha D\left(p(\boldsymbol{z} \mid \boldsymbol{x};\boldsymbol{\theta}),q_\kappa(\boldsymbol{z};\boldsymbol{\phi})\right) \big), 
\label{eq:vib_ii_e_step_sol}  \\
 \pi(\kappa)&=\frac{1}{N}\sum_{i=1}^{N} \pi_{\boldsymbol{x}_i}(\kappa),
\label{eq:prior_cluster_prob_update}
\end{align}
%\end{minipage}
where $\mathcal{Z}_{\boldsymbol{x}}(\alpha)$ is the partition function: $\mathcal{Z}_{\boldsymbol{x}}(\alpha) = \sum_{\kappa=1}^{k} \pi(\kappa) \exp\left(-\alpha D\bigl(p(\boldsymbol{z} \mid \boldsymbol{x};\boldsymbol{\theta}),q_\kappa(\boldsymbol{z};\boldsymbol{\phi}) \bigr)\right)$.

$\pi_{\boldsymbol{x}}$ in~\Eqref{eq:vib_ii_e_step_sol} (see also Eq. 10.124 of~\citet{cover1999elements}) assigns higher probability to the centroid statistically closer in terms of $D$ to the encoder of $\boldsymbol{x}$. $\pi$ in~\Eqref{eq:prior_cluster_prob_update} is derived in Lemma 10.8.1 of~\citet{cover1999elements} and is the marginal of $\pi_{\boldsymbol{x}}$.
%%%%%%%%%%%%%
Steps $t$, $t+2$ require back-propagation and correspond to gradient descent steps. The pseudocode of our method (Algorithm~\ref{alg:training_ig_vib}) along with a practical implementation for mini-batch training is given in Appendix~\ref{sec:training_vib}.
%%%%%%%%%%%%%
%%%%%%%%%%%%%%%%%
\subsection{Uncertainty Quantification in the IB}\label{sec:uncertainty_dab}
%%%%%%
The solution to the problem of~\Eqref{eq:ua_ib} provides us with codes $q_\kappa(\boldsymbol{z};\boldsymbol{\phi})$ for encoders in $\mathcal{D}_{\train}$ (\Figref{d_train_enc_center_codebook}). Large distance from these codes signals an unfamiliar input $\boldsymbol{x}$ for which the network should be less confident when predicting $\boldsymbol{y}$ (\Figref{fig:uncertainty_regression}). Formally, we define uncertainty over datapoint $\boldsymbol{x}$ as the conditional expected distortion (from last term in~\Eqref{eq:ua_ib}):
\begin{align}
& \mathrm{uncertainty}\left(\boldsymbol{x}\right) = \mathrm{distance}\left(\boldsymbol{x},\mathcal{D}_{\train}\right) \nonumber \\ & = \mathbb{E}_{Q\mid P_X=p(\boldsymbol{z} \mid \boldsymbol{x};\boldsymbol{\theta})
}\left[ D\big(p(\boldsymbol{z} \mid \boldsymbol{x};\boldsymbol{\theta}),q_\kappa(\boldsymbol{z};\boldsymbol{\phi})\big)\right].
\label{eq:distance_3}
\end{align}
The distribution of $Q$ in the expectation of~\Eqref{eq:distance_3} conditioned on encoder $p(\boldsymbol{z} \mid \boldsymbol{x};\boldsymbol{\theta})$ (also defined in~\Eqref{eq:conditional_pq}) is given in~\Eqref{eq:vib_ii_e_step_sol}. The expectation in~\Eqref{eq:distance_3} is taken over a finite number of values $k$. Under certain choices of $D$ and distribution families, the uncertainty score of~\Eqref{eq:distance_3} can be computed \textit{deterministically} with a \textit{single forward pass} of the network without requiring Monte Carlo approximations. In this work, we consider the Kullback-Leibler divergence  as the distortion function and multivariate Gaussians for the codes and the encoder.
%%%%%%%%%%%%%
\begin{table*}[!ht]
\centering
\caption{\textbf{OOD performance of baselines trained on the CIFAR-10 dataset.} We consider two OOD datasets for distinguishing from CIFAR-10 with varying levels of difficulty: SVHN (\rebuttal{far} OOD dataset) and CIFAR-100 (\rebuttal{near} OOD dataset). In bold are top results (within standard error). The horizontal line separates ensembles from DUMs. \rebuttal{Only methods with the same background color can be directly compared with each other. }The performance of all models is averaged over 10 random seeds. DAB outperforms all baselines in all tasks with respect to all metrics. \rebuttal{DA stands for distance aware. R indicates whether model has been/ can be applied to regression tasks. PR indicates whether method can be applied to a pre-trained network.}}
\begin{adjustbox}{width=0.99\textwidth,center}
{\Huge
\begin{tabular}{|cc|c|c|c|c|c|c|}
        \hline
\multicolumn{1}{|c|}{\multirow{2}{*}{\textbf{Method}}}    & {\multirow{2}{*}{\textbf{DA}}} & {\multirow{2}{*}{\textbf{R}}}  & {\multirow{2}{*}{\textbf{PR}}}     &
\multicolumn{2}{c|}{\textbf{SVHN}}    & \multicolumn{2}{c|}{\textbf{CIFAR-100}}                     \\ \cline{5-8}
\multicolumn{1}{|c|}{}     &   &  & & \textbf{AUROC} $\uparrow$  & \textbf{AUPRC}   $\uparrow$      & \textbf{AUROC} $\uparrow$ & \textbf{AUPRC}         $\uparrow$                                                                \\ \hline \hline

\rowcolor{amaranth}
\multicolumn{1}{|c|}{\begin{tabular}[c]{@{}c@{}} Deep Ensemble of 5~\citep{lakshminarayanan2017simple}\end{tabular}} & \xmark &  \cmark & \Huge{\textbf{--}} & $0.97\phantom{0}\pm0.004$           &  $0.984\pm0.003$               & $0.916 \pm 0.001$           & 
 $0.902 \pm 0.002$        
  
  \\ \hline \hline 
% \rowcolor{lightkhaki}
% \multicolumn{1}{|c|}{\begin{tabular}[c]{@{}c@{}} \rebuttal{AugMix}~\citep{hendrycks2020augmix} \end{tabular}} & \xmark &  \xmark &  \xmark & \rebuttal{$0.958 \pm 0.006 $}         & \rebuttal{$0.977 \pm 0.004 $}                & \rebuttal{$0.913 \pm 0.002$}             &  \rebuttal{$0.899 \pm 0.003$} \\[-1 pt]
% \hhline{>{\arrayrulecolor{lightkhaki}}----->{\arrayrulecolor{black}}|}\noalign{\vskip-1pt}
% \rowcolor{lightkhaki}
% \multicolumn{1}{|c|}{\begin{tabular}[c]{@{}c@{}} \rebuttal{RegMixup} ~\citep{pinto2022using}\end{tabular}} & \xmark &  \xmark & \xmark & \rebuttal{$0.921 \pm 0.035$}         & \rebuttal{$0.961 \pm 0.018$}                 & \rebuttal{$0.856 \pm 0.014$}             &  \rebuttal{$0.844 \pm 0.013$}\\\hline \hline

% \toprule
\rowcolor{LightCyan}
\multicolumn{1}{|c|}{\begin{tabular}[c]{@{}c@{}}Deterministic~\citep{zagoruyko2016wide}\end{tabular}}    &   \xmark &  \cmark & \Huge{\textbf{--}} & $0.956 \pm 0.004$           & $0.976 \pm 0.004$       & $0.892\pm{0.002}$          &  $0.88\pm 0.002$\\[-1pt] %\hline
\hhline{>{\arrayrulecolor{LightCyan}}----->{\arrayrulecolor{black}}|}\noalign{\vskip-1pt}
\rowcolor{LightCyan}
\multicolumn{1}{|c|}{\begin{tabular}[c]{@{}c@{}} \rebuttal{DDU} ~\citep{Mukhoti_2023_CVPR}\end{tabular}} & \cmark &    \xmark & \xmark & \rebuttal{$0.981 \pm 0.002$}         & \rebuttal{$0.966 \pm 0.003$}                 & \rebuttal{$0.894 \pm 0.001$}             &  \rebuttal{$0.901 \pm 0.001$} \\[-1pt] %\hline
\hhline{>{\arrayrulecolor{LightCyan}}----->{\arrayrulecolor{black}}|}\noalign{\vskip-1pt}
\rowcolor{LightCyan}
\multicolumn{1}{|c|}{\begin{tabular}[c]{@{}c@{}} DUQ~\citep{van2020uncertainty} \end{tabular}}& \cmark &  \xmark & \xmark & $0.940 \pm 0.003$           & $0.956 \pm 0.006$                  &  $0.817 \pm 0.012$            & $0.826 \pm 0.006$ \\[-1pt] %\hline
\hhline{>{\arrayrulecolor{LightCyan}}----->{\arrayrulecolor{black}}|}\noalign{\vskip-1pt}
\rowcolor{LightCyan} 
\multicolumn{1}{|c|}{\begin{tabular}[c]{@{}c@{}} DUE~\citep{van2021feature} \end{tabular}} & \cmark &  \cmark & \xmark & $0.958 \pm 0.005$          & $0.968 \pm 0.015$               & $0.871\pm0.011$           & $0.865\pm0.011$ \\[-1pt] %\hline
\hhline{>{\arrayrulecolor{LightCyan}}----->{\arrayrulecolor{black}}|}\noalign{\vskip-1pt}
\rowcolor{LightCyan}  
\multicolumn{1}{|c|}{\begin{tabular}[c]{@{}c@{}} SNGP~\citep{liu2020simple,JMLR:v24:22-0479} \end{tabular}} & \cmark & \cmark  & \xmark & $0.971 \pm 0.003$           & $0.987 \pm 0.001$               & $0.908\pm0.003$           &  $0.907 \pm 0.002$  \\[-1pt] %\hline
\hhline{>{\arrayrulecolor{LightCyan}}----->{\arrayrulecolor{black}}|}\noalign{\vskip-1pt}
\rowcolor{LightCyan}
\multicolumn{1}{|c|}{\begin{tabular}[c]{@{}c@{}}  vanilla VIB ~\citep{alemi2018uncertainty} \end{tabular}} & \cmark & \cmark &  \cmark & $0.715 \pm 0.081$          & $0.869 \pm 0.039$               & $0.663 \pm 0.045$             & $0.701 \pm 0.034$     \\[-1pt] %\hline
\hhline{>{\arrayrulecolor{LightCyan}}----->{\arrayrulecolor{black}}|}\noalign{\vskip-1pt}
\rowcolor{LightCyan} 
\multicolumn{1}{|c|}{\begin{tabular}[c]{@{}c@{}} \textbf{DAB (ours)} \end{tabular}} & \cmark & \cmark  &  \cmark & $\boldsymbol{0.986 \pm 0.004}$          & $\boldsymbol{0.994 \pm 0.002}$               & $\boldsymbol{0.922 \pm 0.002}$             & $\boldsymbol{0.915 \pm 0.002}$ \\
\bottomrule
\end{tabular}}
\end{adjustbox}
\label{tab:ood_baselines}
\end{table*}
\subsection{Connections with Maximum Likelihood Mixture Estimation.}
%%%%%%%%%%%%%
\begin{table*}[!htbp]
\centering
\caption{\textbf{Accuracy and model size of OOD baselines}. Although we use a narrow bottleneck (8-dimensional latent variables), the accuracy of our model is not compromised compared to other deterministic uncertainty baselines. This is because 10 distributional codes can sufficiently represent the training dataset without diminishing the regularization effect and distance awareness of the rate-distortion constraint. More importantly, DAB can inject uncertainty awareness into the model with a minor model size overhead.}
\centering
\scalebox{0.95}{
\begin{tabular}{|c | c |  c|}
 \hline
 \textbf{Method}  & \textbf{Accuracy $\uparrow$} &  \textbf{\# Trainable Parameters $\downarrow$} \\ [0.5ex] 
 \hline \hline
 \rowcolor{amaranth}
  Deep Ensemble of 5~\citep{lakshminarayanan2017simple} & \textbf{96.6\%}  & $182,395,970$     \\ \hline \hline
\rowcolor{lightkhaki}
% \rebuttal{AugMix}~\citep{hendrycks2020augmix}  & \rebuttal{$96.6 \% $}         & \rebuttal{\textbf{36,479,194}} \\[-1 pt]
% \hhline{>{\arrayrulecolor{lightkhaki}}--->{\arrayrulecolor{black}}|}\noalign{\vskip-1pt} 
% \rowcolor{lightkhaki}
% { \rebuttal{RegMixup} ~\citep{pinto2022using}} & \rebuttal{95.9\%}         & \rebuttal{\textbf{36,479,194}}       
 % \\ \hline \hline
\rowcolor{LightCyan} 
 Deterministic~\citep{zagoruyko2016wide} & $96.2\% $   & \textbf{36,479,194}  \\[-1pt] %\hline
\hhline{>{\arrayrulecolor{LightCyan}}--->{\arrayrulecolor{black}}|}\noalign{\vskip-1pt}
\rowcolor{LightCyan}
 \rowcolor{LightCyan} 
 \rebuttal{DDU~\citep{Mukhoti_2023_CVPR}} & \rebuttal{$95.9\%$}& \rebuttal{\textbf{36,479,194}} \\[-1pt] %\hline
\hhline{>{\arrayrulecolor{LightCyan}}--->{\arrayrulecolor{black}}|}\noalign{\vskip-1pt}
\rowcolor{LightCyan}
  DUQ~\citep{van2020uncertainty} & 94.9\% & 40,568,784  \\[-1pt] %\hline
\hhline{>{\arrayrulecolor{LightCyan}}--->{\arrayrulecolor{black}}|}\noalign{\vskip-1pt}
\rowcolor{LightCyan}
 DUE~\citep{van2021feature} & 95.6\% & 36,480,314 \\[-1pt] %\hline
\hhline{>{\arrayrulecolor{LightCyan}}--->{\arrayrulecolor{black}}|}\noalign{\vskip-1pt}
\rowcolor{LightCyan}
SNGP~\citep{liu2020simple, JMLR:v24:22-0479} & 95.9\% & 36,483,024  \\[-1pt] %\hline
\hhline{>{\arrayrulecolor{LightCyan}}--->{\arrayrulecolor{black}}|}\noalign{\vskip-1pt}
\rowcolor{LightCyan}  
vanilla VIB~\citep{alemi2018uncertainty} & 95.9\%  & 36,501,042 \\[-1pt] %\hline
\hhline{>{\arrayrulecolor{LightCyan}}--->{\arrayrulecolor{black}}|}\noalign{\vskip-1pt}
\rowcolor{LightCyan}
DAB (ours)  & 95.9\% & 36,501,114  \\ [1ex]
\bottomrule
\end{tabular}}
\label{table:OoD_params_accuracy}
\end{table*}
%%%%%%%%%%%%%%%
Limited work has sought connections between Maximum Likelihood
Mixture Estimation (MLME) and computation of the rate-distortion function.~\citet{banerjee2004information} prove the equivalence between these two problems for Bregman distortions and exponential families. In this case and under the assumption of constant variance for all mixture's components, learning the support set in RDFC corresponds to learning the mixture means. For MLME on parametric distributions, i.e., encoders, a straightforward way to leverage this connection is to define the ``sample space'' of the MLME as the ``parameter space'' of encoder's distribution family. Similarly, training with a mixture (for the marginal) VIB~\citep{alemi2018uncertainty} entails an MLME problem where the data points (to be clustered) are latent samples drawn from encoders. To get better insights, in Appendix~\ref{sec:vib_for_clustering} we anatomize the loss function.  As we will see in Table~\ref{tab:ood_baselines}, a full statistical description of encoders (instead of using a finite--single in the experiment-- number of its samples) along with the proposed alternating minimization algorithm that guides assignments to centroids during training, helps DAB capture uncertainty \textit{exactly} with a \textit{single} forward pass. From a theoretical standpoint,  deriving rigorous connections between the two problems would be interesting for future work.

%% file: experiments.tex
%auto-ignore
%%%%%%%%%%%%%%%%%%%%%%%
\section{Experiments}\label{sec:experiments}
%%%%%
%%%%%%
\subsection{Synthetic Example}\label{sec:regression_experiments}
%%%%%%
Before we compare with other DUMs, we first need to sanity-check the proposed model and learning algorithm. Synthetic experiments are handy for this task since they allow us to test the behavior of the model under different conditions. In this work, we apply DAB to synthetic regression tasks. In \Figref{fig:uncertainty_regression}, we visualize the predictive uncertainty, i.e., the value of the distortion function in~\Eqref{eq:distance_3}.  We verify that as we move far away from the data, the model's confidence and accuracy decline. We consider two cases of training datasets. \Figref{fig:regression_uniform} follows the original set-up of~\citet{hernandez2015probabilistic}.  \Figref{fig:regression_clusters} is a harder variant of the first problem~\citep{foong2019between} and a typical failure case of many uncertainty-aware methods. ~\citet{wilson2020bayesian} show that many methods end up being overconfident in the area between the clusters of the training datapoints. We provide details for the dataset generation and the training setup in Appendix~\ref{sec:synthetic_setup}.
%%%%%%%%%%%%%%%%%%%%%%%
\subsection{DAB for Out-of-Distribution Detection}\label{sec:ood_experiments}
To compare the uncertainty quality of different models, we evaluate their performance in distinguishing between the test sets of CIFAR-10 and OOD datasets. \rebuttal{We consider two OOD datasets of increasing difficulty: SVHN~\citep{svhn} (far OOD/ easy task) and CIFAR-100 (near OOD/ difficult task).} We compare DAB against a deterministic baseline, 
an ensemble baseline, a VIB with a mixture marginal trained with gradient descent, and the most competitive DUMs. All approaches do not require auxiliary OOD datasets either to train the models or to tune hyperparameters. \rebuttal{In Table~\ref{tab:ood_baselines}, we also outline some  high-level properties of these models.}  For all methods, we use Wide ResNet 28-10~\citep{zagoruyko2016wide} as the base network. DAB and VIB are inserted right before output's dense layer. For both, we use 8-dimensional latent features. For DAB, we consider $k=10$ centroids. For VIB, we consider a mixture with $10$ components. We use the Kullback-Leibler divergence as the distortion function in~\Eqref{eq:distance_3}. For fair comparisons, we train the IB and the Gaussian Process models with a single sample. Further training and evaluation configurations are given in Appendix~\ref{sec:setup_ood_experiments}. 

As shown in Table~\ref{tab:ood_baselines}, DAB outperforms all baselines in terms of AUPRC and AUROC (the positive label in the binary OOD classification corresponds to the OOD images). We confirm that distances in distribution space are more informative compared to Euclidean distances. In Table~\ref{table:OoD_params_accuracy}, we report the accuracy and the size of the baselines. We note here that the accuracy of our model is on par with that of other DUMs. Importantly, DAB only minimally increases the single network's size while rendering it uncertainty-aware. The additional parameters correspond to centroids’ parameters and DAB's dense layers implementing the head of the encoder.
%%%%%%%%%%%%%%%
\begin{figure*}[!htbp]
\centering
\caption{\textbf{Qualitative evaluation of encoders' codebook.} We visualize the number of CIFAR-10 test data points per class assigned to each centroid during training. We assign a data point to the centroid with the smallest statistical distance from its encoder. Each centroid progressively attracts data points of the same class. Moreover, all centroids are assigned a non-zero number of test datapoints. Therefore, the centroids are useful for better explaining both train and previously unseen, test data points.}
\label{fig:10_clusters}
\includegraphics[width=0.70\textwidth]{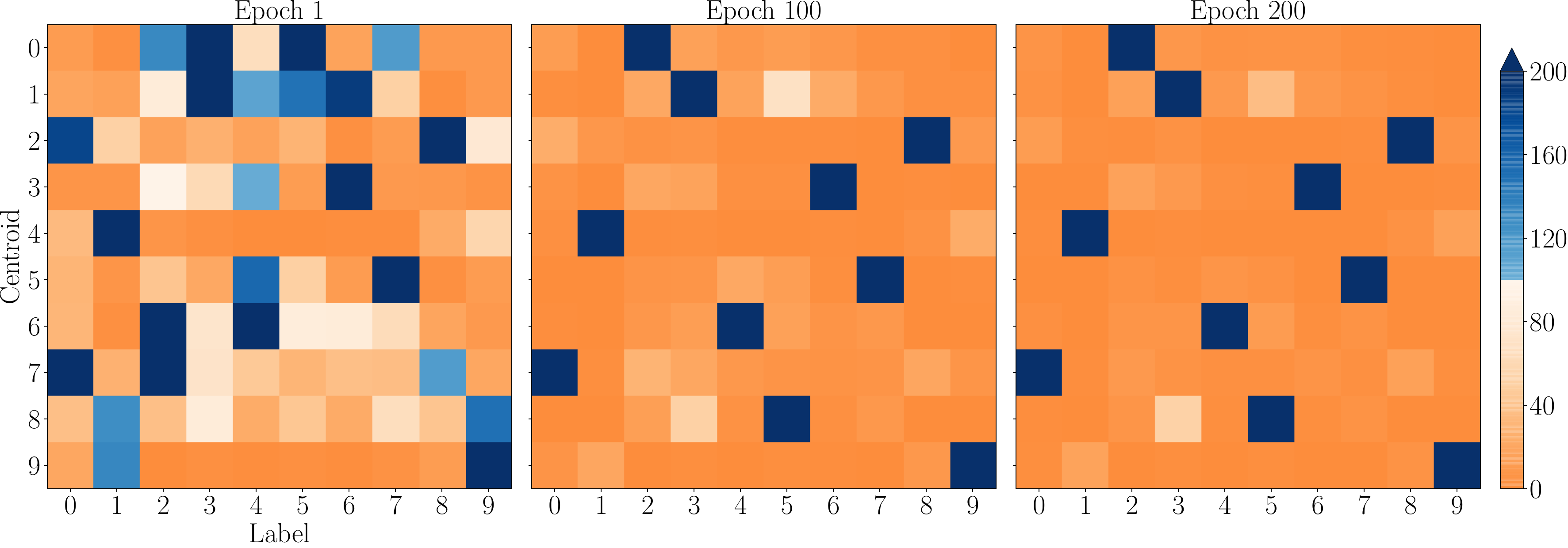}
\end{figure*}
%%%%%%%%%%%%%%%

In Figure~\ref{fig:dab_distance_awareness}, we visualize distances from the learned codebook. The in-distribution, test datapoints that are correctly classified lie within the statistical balls (\Figref{d_train_enc_center}) defined by codebook's centroids and $\mathcal{D}_\train$. The in-distribution, misclassified test datapoints are clearly separated from the training support but closer to the codebook than the near OOD. This, as we will see in the next section, qualitatively justifies DAB's strong calibration (Tables~\ref{table:calibration_score_},~\ref{table:imagenet}). Lastly, near OOD datapoints  are closer to the codebook than far OOD datapoints.

To qualitatively inspect the learning algorithm of Section~\ref{sec:learning_algorithm}, in \Figref{fig:10_clusters} we plot the number of test datapoints per class represented by each centroid $q_\kappa(\boldsymbol{z};\boldsymbol{\phi})$ in different training phases. We note that the class counts refer to the \textit{true} and not the \textit{predicted} class. As training proceeds, DAB learns similar latent features for datapoints that belong to the same class and pushes them closer to the same centroid. Certain centroids, however, conflate test datapoints of different classes. For example, a small number of test datapoints of class 3 (cat) are assigned to (are closest to) the centroid whose majority class is 5 (dog). Assignment to the wrong centroid presages model's misprediction for these datapoints. In contrast, we observed in an analogous figure using the training data points that they are completely separated by the centroids (the colormap displays only blue squares). However, this might not be the case in training datasets containing corrupted labels~\citep{northcutt2021confident}.

%%%%%%%%%%%%%
\begin{table*}[!ht]
\caption{\textbf{Calibration AUROC of DUMs for misclassification prediction on CIFAR-10.} We examine how well the model can predict it will be wrong from its estimated uncertainty. The problem is framed as a binary classification task with the positive label indicating a mistake. DAB closes the gap between DUMs and ensembles~\citep{pmlr-v162-postels22a}.}
\centering
\begin{adjustbox}{width=0.8\textwidth,center}
{\Huge
\begin{tabular}{|c | c | c|} 
 \hline
 \textbf{Method}  &  \textbf{ Uncertainty Description} &   \textbf{Calibration AUROC $\uparrow$} \\ [0.5ex] 
 \hline \hline
 \cellcolor{amaranth} Deep Ensemble of 5~\citep{lakshminarayanan2017simple} &  \cellcolor{amaranth}  Gibbs softmax entropy &  \cellcolor{amaranth}$\boldsymbol{0.951 \pm 0.001}$   \\ \hhline{>{\arrayrulecolor{amaranth}}->{\arrayrulecolor{black}}->{\arrayrulecolor{black}}-}%\noalign{\vskip-1pt}
 \hline \hline
\rowcolor{LightCyan}
   \rowcolor{LightCyan}
   \rebuttal{DDU~\citep{Mukhoti_2023_CVPR}} & \rebuttal{ Softmax entropy}  & \rebuttal{$0.632 \pm 0.009$}  \\[-1pt] %\hline
\hhline{>{\arrayrulecolor{LightCyan}}--->{\arrayrulecolor{black}}|}\noalign{\vskip-1pt} 
    \rowcolor{LightCyan}
  DUQ~\citep{van2020uncertainty} &  Euclidean distance ($l_2$-norm) from centroid  & $0.889 \pm 0.013$  \\[-1pt] %\hline
\hhline{>{\arrayrulecolor{LightCyan}}--->{\arrayrulecolor{black}}|}\noalign{\vskip-1pt}
   \rowcolor{LightCyan}
  DUE~\citep{van2021feature} &  Posterior variance & $0.856 \pm 0.026$ \\[-1pt] %\hline
\hhline{>{\arrayrulecolor{LightCyan}}--->{\arrayrulecolor{black}}|}\noalign{\vskip-1pt}
  \rowcolor{LightCyan}
 SNGP~\citep{liu2020simple,JMLR:v24:22-0479} &  Dempster-Shafer uncertainty &  $0.897 \pm 0.006$   \\[-1pt] %\hline
\hhline{>{\arrayrulecolor{LightCyan}}--->{\arrayrulecolor{black}}|}\noalign{\vskip-1pt}  
 \rowcolor{LightCyan}
 \textbf{DAB (ours)}  &  \textbf{Statistical distance (KL) from centroid}  &    $\boldsymbol{ 0.930 \pm  0.003}$  \\ [1ex] 
 \hline
\end{tabular}}
\end{adjustbox}
\label{table:calibration_score_}
\end{table*}
%%%%%%%%%%%%%%%
\begin{table*}[!ht]
\caption{\textbf{\rebuttal{DAB's performance on ImageNet-1K.}} DAB outperforms ensembles at predicting misclassifications. Moreover, it can better distinguish ImageNet-O from ImageNet images. More importantly, it does so with significantly fewer trainable parameters. The performance of all models is averaged over 4 random seeds.
} 
\centering
\begin{adjustbox}{width=0.99\textwidth,center}
%\midsepremove
{\Huge
\begin{tabular}{ |c  | c | c | c|  c| c| }
 \hline
  \textbf{Method}  & \textbf{Uncertainty Description}  & \textbf{Calibration AUROC} $\uparrow$ &  \textbf{\thead{\Huge OOD AUROC \\ \Huge ImageNet-O}} $\uparrow$ & \textbf{Accuracy} $\uparrow$ &\textbf{\thead{ \Huge\# Trainable \\ \Huge Parameters }}      \\ %[0.5ex] 
 \hline \hline
\rowcolor{amaranth}
Deep Ensemble of 5 \citep{lakshminarayanan2017simple} &  \cellcolor{amaranth} Gibbs softmax entropy & \cellcolor{amaranth} $0.861 \pm 0.0004$ & $0.642 \pm 0.001$  & \cellcolor{amaranth} $\boldsymbol{78.4 \pm 0.06\%}$ & \cellcolor{amaranth}  $117,672,960$\\
\hline \hline
\rowcolor{LightCyan}
% Deterministic~\citep{zagoruyko2016wide} & Softmax entropy & 0.851 & xxxx   &   &    \\ \hline \hline
% \rowcolor{LightCyan}
% \textbf{DAB (ours)} & Statistical distance (KL) & $\boldsymbol{0.867}$ & $\boldsymbol{0.739}$   & $76.1\%$ & $\boldsymbol{28,056,808}$ 
\rowcolor{LightCyan}
\textbf{DAB with fine-tuned ResNet-50 (ours)} & Statistical distance (KL) & $\boldsymbol{0.868 \pm 0.0008} $ & $\boldsymbol{0.743 \pm 0.004}$   & $76.1  \pm 0.02\%$ & $36,612,328$   \\ 
\rowcolor{LightCyan}
DAB with pre-trained ResNet-50 (ours) & Statistical distance (KL) & $0.866 \pm 0.0003 $ & $0.732 \pm 0.004$   & $74.71  \pm 0.09\%$ & $\boldsymbol{13,077,736}$   \\
\bottomrule
\end{tabular}}
\end{adjustbox}
\label{table:imagenet}
\end{table*}
%%%%%%%%%%%%%%%%%%%%%%%%%%%%%%%%%%%%%
\subsection{DAB for Misclassification Prediction}\label{sec:calibration_experiments}
\rebuttal{To further assess the quality of the proposed distance score (\Eqref{eq:distance_3}), we evaluate DAB's performance on misclassification prediction\rebuttal{~\citep{corbiere2019addressing, zhu2022rethinking}}. Misclassification prediction is formulated as a binary classification task with the positive label indicating a classifier's mistake.} We report the \textit{Calibration AUROC} that was introduced by~\citet{50669} and later used by~\citet{pmlr-v162-postels22a}. As pointed out by~\citet{pmlr-v162-postels22a}, the ECE (Expected Calibration Error) is not the appropriate metric for DUMs since their uncertainty scores are not directly reflected to the probabilistic forecast. 
Another benefit of Calibration AUROC compared to ECE is that it cannot be trivially reduced using 
post hoc calibration heuristics such as
temperature scaling~\citep{guo2017calibration}. In contrast, Calibration AUROC focuses on the intrinsic ability of the model to distinguish its correct from incorrect predictions and the ranking performance of its uncertainty score, i.e., whether high uncertainty predictions are wrong. %We remark that such property can be particularly useful in unsupervised learning tasks~\citep{osband2016deep, lee2021sunrise, mai2022sample}.
%%%%%
%%%%%%%%%%%%%
In Table~\ref{table:calibration_score_}, we first evaluate DUMs' performance in predicting misclassified CIFAR-10 images. Here, we note that DAB bridges the gap between baselines and costly ensembles.

\rebuttal{To illustrate scalability, we focus on the large-scale ImageNet dataset~\citep{russakovsky2015imagenet} for the rest of this section. We observe that previous DUMs either exhibit training instability issues when scaled to larger datasets or fall behind in calibration~\citep{pmlr-v162-postels22a}. For this experiment, we use the ResNet-50 architecture. For DAB, we instantiate the backbone network with the publicly available, pre-trained weights (excluding the last dense layer of the classifier). The ResNet-50 features are passed through three fully connected dense layers that produce DAB's input. We consider two cases. First, we further fine-tune ResNet-50 alongside DAB. Next, we consider a setup similar to that of~\citet{alemi2016deep} where gradients are not backpropagated to the backbone network. This substantially decreases the training time and the number of trainable parameters. In both cases, we train DAB for 70 epochs. DAB uses a codebook with $1000$ entries and 80-dimensional latent features. The implementation details are deferred to Appendix~\ref{sec:imagenet_setup}}. 
%
% without substantially hurting DAB's calibration or OOD detection,
% In Table~\ref{table:imagenet} of Section~\ref{sec:calibration_experiments}, the pre-trained ResNet-50 was further fine-tuned alongside DAB. This setup is slightly different than that of~\citet{alemi2016deep} where gradients were not backpropagated to the backbone network. Here, we make similar considerations with~\citet{alemi2016deep}. This substantially decreases the training time and the number of trainable parameters without substantially hurting DAB's calibration or OOD detection, as we see in Table~\ref{imagenet_pretrained}. The accuracy of the pretrained ResNet-50  is $74.9$\footnote{\url{https://keras.io/api/applications/}}. We see that DAB almost matches this accuracy similar to the standard VIB on top of the pre-trained network.
% A similar pipeline was considered for VIB by~\citet{alemi2016deep}. 

We leverage DAB's distance awareness and consider a variant of the learning algorithm presented in Section~\ref{sec:learning_algorithm}. 
In particular, we modify the training objective in~\Eqref{eq:ua_ib} to encourage high uncertainty for the misclassified datapoints in $\mathcal{D}_\train$. This is achieved by adding a max-margin loss term (\Eqref{eq:u_lb}) in the objective at the gradient updates (Algorithm~\ref{alg:training_ig_vib}) to push the misclassified datapoints in $\mathcal{D}_\train$ away from the codebook. The codebook is trained to represent only the correctly classified training examples.\footnote{This variant has no effect for the CIFAR-10 experiments. This is because all models achieve very high training accuracy very early, eventually reaching $100\%$, leaving no misclassified examples for the repulsive loss term.}We notice that: 

\textit{penalizing high or small uncertainty (distance from the codebook) for the training examples according to the classification outcome improves model's calibration on the test examples.}

\rebuttal{For completeness, we also examine DAB's performance on the ImageNet vs ImageNet-O~\citep{hendrycks2021nae} OOD task. For the OOD experiments, we quantize all training datapoints regardless the classification outcome. We report only AUROC which is preferred in situations of highly imbalanced OOD tasks~\citep{pinto2022impartial} -- ImageNet-O has only $2,000$ OOD images. %Finally, we note that the model achieves an accuracy of $76.1\%$ matching the accuracy of a single DNN.
} 

In Table~\ref{table:imagenet}, we report performance against ensembles which is the gold standard in calibration and OOD detection. As we see, DAB has better calibration and OOD detection than ensembles in both cases. We remark that applying DAB without ResNet-50  fine-tuning does not substantially hurt its calibration or OOD capability. The small performance gap is attributed to the fact that the largest part of the encoder is not regularized to stay close to the codebook. Finally, we see that DAB nearly reaches the initial accuracy of $74.9$ achieved by the pre-trained ResNet-50 \footnote{\url{https://keras.io/api/applications/}} like the standard VIB~\citep{alemi2016deep}.

\textbf{Additional Experiments.} Due to space constraints, we supplement the experiments in Appendix. We ablate DAB's hyperparameters in Appendix~\ref{sec:ablation}. Appendix~\ref{sec:corruption} evaluates DAB on corrupted CIFAR-10. Appendix~\ref{sec:visualize_clusters} provides further qualitative evaluations of the learned codebook. In Appendix~\ref{sec:regression_ood}, we test DAB on OOD regression problems.
\vspace{-0.01pt}

%% file: conclusion.tex
\section{Limitations \& Future Research}
The main purpose of this work is to define and analyze a more comprehensive notion of distance from the training data manifold under the auspices of information bottleneck methods. Although in the experiments we used the Kullback-Leibler divergence, the proposed framework is flexible and supports inference with alternative statistical distances~\citep{minka2005divergence,nielsen2023fisher}. Evaluating the impact of diverse distance metrics on model's performance is a compelling avenue for future work.

DAB, like other DUMs, currently falls behind ensembles in terms of accuracy. As we briefly discussed in Section~\ref{sec:ua_ib}, it remains to be seen whether this can be fixed by redesigning DAB's decoder to make use of its distance score. In this article, DAB was demonstrated primarily on image classification tasks. Applying DAB in different settings such as natural language generation~\citep{xiao-etal-2022-uncertainty} is another important application area. In this work, we did not use additional OOD datasets during training. DAB's Outlier Exposure (OE)~\citep{hendrycks2019oe} by repelling OE datapoints away from the codebook could further improve OOD capability. Moreover, leveraging the majority vote among data points within each centroid (\Figref{fig:10_clusters}) can enhance the model's ability to make accurate predictions, even when faced with labels containing errors~\citep{platanios2020learning}. Finally, analyzing DAB in concert with data augmentation methods for enhancing the codebook for image datasets is another interesting line of future research. We intend this paper to offer a fresh perspective on uncertainty estimation and we believe its empirical findings are an important step toward future directions mentioned above. 
%%%%%%%%%%%%%%%
\section{Conclusion}
%%%%%%%%%%%%%%%
We introduced DAB, \rebuttal{a distance-aware framework} for deep neural networks (DNNs). We framed \rebuttal{distance awareness} as a rate-distortion problem to learn a lossy compression of the training dataset via a codebook of encoders. %We used the Information Bottleneck formalism where the compression (regularization) term acts on the encoder's density function and the distortion function corresponds to a statistical distance. We designed a practical learning algorithm that is based on successive refinement of the rate-distortion function estimates. 
%Distance from this codebook corresponds to a statistical distance. 
Experimental analysis shows that DNNs equipped with distances from this codebook outperform expensive baselines at OOD tasks and are better calibrated.%.
%%%%%%%%%%%%%%%%%%%%

%% file: appendix.tex
%auto-ignore
\appendix
%\section{Appendix}
\section{Preliminaries}\label{appendix:preliminaries}
\subsection{Definitions}
\begin{definition}[\textit{\textbf{Bregman Divergence}}]\label{def:bregman_divergence}
Let $f: \mathcal{S}\rightarrow \mathbb{R}$ be a differentiable, strictly convex function of Legendre type on a convex set $\mathcal{S}\subseteq\mathbb{R}^d$. The \textit{Bregman divergence} $D_{f}:\mathcal{S} \times \mathcal{S} \rightarrow [0, \infty)$  for any two points $\boldsymbol{x},\boldsymbol{y} \in \mathcal{S}$ is defined as~\citep{bregman1967relaxation}:
\begin{align}
D_{f}(\boldsymbol{x},\boldsymbol{y})=f(\boldsymbol{x})-f(\boldsymbol{y}) - \langle \boldsymbol{x}-\boldsymbol{y},\nabla f (\boldsymbol{y})\rangle,
\end{align}
where $\nabla f(\boldsymbol{y}) \in \mathbb{R}^d$ denotes the gradient vector of $f$ evaluated at $\boldsymbol{y}$.
\end{definition}
%%%%%%%%%%%%%%
\begin{definition}[\textit{\textbf{Dual Bregman Form of Exponential Family}}]\label{def:bregman_form_exponential}
Each probability density function for $\boldsymbol{x}\in  \mathcal{X} \subseteq \mathbb{R}^d$ in the exponential family $\mathcal{F}_{\psi}=\{p_{\psi}(\cdot ; {\boldsymbol{\phi}})\mid \boldsymbol{\phi}\in \Phi\}$, where $\Phi=\mathrm{dom}(\psi) \subseteq \mathbb{R}^p$,  
has the form:
\begin{align}
p_{\psi}(\boldsymbol{x} ;\boldsymbol{\phi} )=\exp(\langle \boldsymbol{t}(\boldsymbol{x}),\boldsymbol{\phi}\rangle -\psi(\boldsymbol{\phi}))\, h_0(\boldsymbol{x}).
\label{eq:exp_form}
\end{align}
$\boldsymbol{t}(\boldsymbol{x})$ is the natural statistic of the family. $\boldsymbol{\phi}$ is called the natural parameter and $\Phi$ the natural parameter space. $\psi(\boldsymbol{\phi})$ is the log-partition function of the family that normalizes the density function. $h_0(\boldsymbol{x})$ is a non-negative function that does not depend on $\boldsymbol{\phi}$. 

If $\boldsymbol{t}(\boldsymbol{x})$ is minimal, i.e., $\nexists$ non-zero $\boldsymbol{\alpha}\in \mathbb{R}^p$ such that $\langle \boldsymbol{\alpha}, \boldsymbol{t}(\boldsymbol{x})\rangle=c$ (a constant) $\forall \boldsymbol{x}\in \mathcal{X}$, and $\Phi$ is open, i.e., $\Phi=\mathrm{int}(\Phi)$, then $\mathcal{F}_{\psi}$ is called \textit{regular exponential family}. In this case, it can be shown~\citep{barndorff2014information} that $\Phi$ is a non-empty convex set in $\mathbb{R}^d$ and that $\psi$ is a convex function. From Theorem 4 by~\citet{banerjee2005clustering}, the density of~\Eqref{eq:exp_form} can be written as:
\begin{align}
p_{\psi}(\boldsymbol{x} ; \boldsymbol{\phi})  &= \exp\left(-D_{\psi^*}\big(\boldsymbol{t}(\boldsymbol{x}),\that{\boldsymbol{t}}(\boldsymbol{\phi})\big)\right)f_{\psi^*}(\boldsymbol{x}), \label{eq:exp_bregman_forma}
\end{align}
where $\psi^*$ is the Legendre-conjugate of $\psi$ and $D_{\psi^*}$ the corresponding Bregman divergence (def.~\ref{def:bregman_divergence}). $\that {\boldsymbol{t}}(\boldsymbol{\phi})$ is the expectation of the sufficient statistic:
\begin{align}
&\that {\boldsymbol{t}} (\boldsymbol{\phi}) \triangleq\mathbb{E}_{X}  [\boldsymbol{t}(\boldsymbol{x})].\label{eq:mean_t}
\end{align}
By differentiating $\int p_{\psi}(\boldsymbol{x};\boldsymbol{\phi}) d\boldsymbol{x}=1$ with respect to $\boldsymbol{\phi}$ and by making use of~\Eqref{eq:exp_form} and~\Eqref{eq:mean_t}, it can be proved that:
\begin{align}
\that{\boldsymbol{t}}(\boldsymbol{\phi})  = \nabla \psi (\boldsymbol{\phi}).
\label{eq:mean_t_b}
\end{align}
Finally, $f_{\psi^*}(\boldsymbol{x})$ is a non-negative function that does not depend on $\boldsymbol{\phi}$:
\begin{align}
f_{\psi^*}(\boldsymbol{x})  =\exp\left(\psi^*(\boldsymbol{t}(\boldsymbol{x}))\right)h_0(\boldsymbol{x}) \label{eq:exp_bregman_formb}.
\end{align}
% \end{subequations}
Therefore, when we train by Maximum Likelihood Estimation (MLE) to learn $\boldsymbol{\phi}$, this term can be omitted from the objective function. \Eqref{eq:exp_bregman_forma} is called the \textit{Bregman form} of the exponential family (\Eqref{eq:exp_form}) and provides a convenient way to parametrize the exponential family
distribution with its expectation parameter (\Eqref{eq:mean_t}).
\end{definition}

\begin{definition}[\textit{\textbf{Scaled Exponential Family}}]\label{def:scaled_exp_family}
Given an exponential family $\mathcal{F}_{\psi}$ with natural parameter $\boldsymbol{\phi}$ and log-partition function $\psi(\boldsymbol{\phi})$ (\Eqref{eq:exp_form}), a \textit{scaled exponential family}~\citep{jiang2012small} $\mathcal{F}^\alpha_{\psi}$ with $\alpha>0$ has natural parameter $\tilde{\boldsymbol{\phi}}=\alpha \boldsymbol{\phi}$ and log-partition function $\tilde{\psi}(\tilde{\boldsymbol{\phi}})=\alpha \psi(\tilde{\boldsymbol{\phi}}/\alpha)=\alpha \psi(\boldsymbol{\phi})$. In case $\mathcal{F}_{\psi}$ is a regular exponential family, the Bregman form of the scaled family is~\citep{jiang2012small}:
\begin{align}
p_{\tilde{\psi}}(\boldsymbol{x};\tilde{\boldsymbol{\phi}}) = \exp\left(-\alpha D_{\psi^*}\big(\boldsymbol{t}(\boldsymbol{x}),\that {\boldsymbol{t}}(\boldsymbol{\phi})\big)\right)f_{\alpha \psi^*}(\boldsymbol{x}),
\end{align}
\end{definition}
where $\psi^*$ is the Legendre-conjugate of $\psi$. $f_{\alpha \psi^*}$ is defined in~\Eqref{eq:exp_bregman_formb} where we scale $\psi^*$ by $\alpha$. Finally, the mean $\that {\boldsymbol{t}}(\boldsymbol{\phi})$ of $\mathcal{F}^\alpha_{\psi}$ is the same with that of $\mathcal{F}_{\psi}$ and is given in~\Eqref{eq:mean_t},~\Eqref{eq:mean_t_b}.

\subsection{Variational Information Bottleneck}
\citet{alemi2016deep} derive efficient variational estimates of the mutual information terms in~\Eqref{eq:ib_ii}. The accuracy term is:
\begin{align}
I(Z,Y;\boldsymbol{\theta}) &= \int \log \frac{p(\boldsymbol{y}\mid \boldsymbol{z} ; \boldsymbol{\theta})}{p(\boldsymbol{y}) }p(\boldsymbol{y},\boldsymbol{z};\boldsymbol{\theta})d\boldsymbol{z}d\boldsymbol{y}. \footnotemark
\label{eq:accuracy_i}
\end{align} 
The decoder $p(\boldsymbol{y}\mid \boldsymbol{z} ;\boldsymbol{\theta})$ in~\Eqref{eq:accuracy_i} is fully defined:
\begin{align}
p(\boldsymbol{y}\mid \boldsymbol{z} ;\boldsymbol{\theta})= \int \frac{p(\boldsymbol{y}|\boldsymbol{x})p(\boldsymbol{z}|\boldsymbol{x};\boldsymbol{\theta})p(\boldsymbol{x})}{p(\boldsymbol{z};\boldsymbol{\theta})} d\boldsymbol{x}.
\label{eq:exact_decoder}
\end{align}
\footnotetext{Note that in our problem, $p(\boldsymbol{y}) ,p(\boldsymbol{x}), p(\boldsymbol{y}\mid \boldsymbol{x})$ refer to our training dataset. Therefore, they are independent of $\boldsymbol{\theta}$.}
Generally,~\Eqref{eq:exact_decoder} cannot be computed in closed-form. Moreover, it contains the intractable marginal $p(\boldsymbol{z};\boldsymbol{\theta})$:
\begin{align}
p(\boldsymbol{z};\boldsymbol{\theta})=\int p(\boldsymbol{z}\mid\boldsymbol{x};\boldsymbol{\theta})p(\boldsymbol{x})d\boldsymbol{x}.
\label{eq:exact_marginal}
\end{align}
Similarly, the regularization term is analytically intractable since:
\begin{align}
I(Z,X;\boldsymbol{\theta}) &= \int \log \frac{p(\boldsymbol{z}\mid \boldsymbol{x};\boldsymbol{\theta})}{p(\boldsymbol{z};\boldsymbol{\theta})}p(\boldsymbol{x},\boldsymbol{z};\boldsymbol{\theta})d\boldsymbol{z}d\boldsymbol{x}
\label{eq:regularization_i}.
\end{align}

Variational estimates in a distributional family $m(\boldsymbol{y}\mid \boldsymbol{z};\boldsymbol{\theta})$ \footnote{As in the main paper, we use $\boldsymbol{\theta}$ to denote the joint set of parameters of both encoder and variational decoder.} and $q(\boldsymbol{z};\boldsymbol{\phi})$ of~\Eqref{eq:exact_decoder},~\Eqref{eq:exact_marginal} minimize the Kullback-Leibler divergences $D_\KL(p(\boldsymbol{y}\mid \boldsymbol{z} ;\boldsymbol{\theta}), m(\boldsymbol{y}\mid \boldsymbol{z};\boldsymbol{\theta}))$ and $D_\KL(p( \boldsymbol{z} ;\boldsymbol{\theta}), q(\boldsymbol{z};\boldsymbol{\phi}))$, respectively. Non-negativity of the Kullback-Leibler divergence yields a lower bound of~\Eqref{eq:accuracy_i} and an upper bound of~\Eqref{eq:regularization_i}. Substituting these variational bounds in~\Eqref{eq:ib_ii} gives us the Variational Information Bottleneck (VIB) minimization loss:
\begin{align}
\mathcal{L}_{\mathrm{VIB}} = \mathbb{E}_{X,Y,Z}[-\log m(\boldsymbol{y}\mid \boldsymbol{z};\boldsymbol{\theta})]+\beta \mathbb{E}_{X}[D_\KL(p(\boldsymbol{z}\mid \boldsymbol{x};\boldsymbol{\theta}),q(\boldsymbol{z};\boldsymbol{\phi}))].
\label{eq:vib_loss}
\end{align}

\section{Learning Algorithm (Section~\ref{sec:learning_algorithm} continued.)}\label{sec:training_vib}
\RestyleAlgo{ruled}
\SetKwInput{KwData}{Inputs}
\SetKwInput{KwResult}{Outputs}
\SetKwInput{KwInit}{Initialize}
\SetKwProg{Initialize}{init}{}{}

\SetKwComment{Comment}{/* }{ */}
\begin{algorithm}[ht]
\caption{Optimization of Distance Aware Bottleneck}\label{alg:training_ig_vib}
\KwData \\
\ \ training data: $\mathcal{D}_{\train}=\{(\mathbf{x}_i,\mathbf{y}_i)\}_{i=1}^{N}$ \\
\ \ codebook size: $k$ \\ 
\ \ statistical distance: $D$ \\ 
  \ \ hyper-parameters: \\
  \ \ \ \ regularization coefficient $\beta \ge 0$ (\Eqref{eq:ua_ib})\\ 
  \ \ \ \ temperature $\alpha \ge 0 $ (\Eqref{eq:vib_ii_e_step_sol}) 
 \\ 
\KwResult \\
\ \ optimal parameters of encoder and decoder: $\boldsymbol{\theta}^*$ \\
\ \ optimal codebook parameters: $\boldsymbol{\phi}^*$ \\ 
\ \ marginal assignment probabilities: $\pi*$ \\ 
\KwInit \\ 
\ \ encoder $p_{}(\boldsymbol{z} \mid \boldsymbol{x};\boldsymbol{\theta})$ \\
\ \ decoder  \ $m(\boldsymbol{y} \mid \boldsymbol{z};\boldsymbol{\theta})$ \\ 
\ \ codebook $\{q_{\kappa}(\boldsymbol{z};\boldsymbol{\phi})\}_{\kappa=1}^{k}$ \\
\ \ $\pi$, $\pi_{{\boldsymbol{x}_i}}$  to uniform 
distribution \\
\While{not converged} {
%\phantom{empty line}\\
\textbf{step 1:} \\
Update decoder $m(\boldsymbol{y} \mid \boldsymbol{z};\boldsymbol{\theta})$, encoder $p_{}(\boldsymbol{z} \mid \boldsymbol{x};\boldsymbol{\theta})$: $\boldsymbol{\theta} \leftarrow \boldsymbol{\theta} - \eta_{\boldsymbol{\theta}} \nabla_{\boldsymbol{\theta}} \mathcal{L}_{\mathrm{DAB}}$ \  ($\mathcal{L}_{\mathrm{DAB}}$ in~\Eqref{eq:ua_ib})\\ 
%\phantom{empty line}\\
\textbf{step 2:} \\
\For{$i=1,2,\dots,N$}{
\For{$\kappa=1,2,\dots,k$}{
$\pi_{\boldsymbol{x}_i}(\kappa) = \frac{\pi(\kappa)}{\mathcal{Z}_{\boldsymbol{x}_i}(\alpha)}\exp\left(-\alpha D(p(\boldsymbol{z} \mid \boldsymbol{x};\boldsymbol{\theta}),q_\kappa(\boldsymbol{z};\boldsymbol{\phi}))\right)$ \ 
(\Eqref{eq:vib_ii_e_step_sol})
}
}
%\phantom{empty line}\\
\textbf{step 3:} \\
Update codes $q_{\kappa}(\boldsymbol{z}; \boldsymbol{\phi})$: $\boldsymbol{\phi} \leftarrow \boldsymbol{\phi} - \eta_{\boldsymbol{\phi}} \nabla_{\boldsymbol{\phi}} \mathcal{L}_{\mathrm{DAB}}$ \ ($\mathcal{L}_{\mathrm{DAB}}$ in~\Eqref{eq:ua_ib})\\
%\phantom{empty line}\\
\textbf{step 4:} \\
    \For{$\kappa=1,2,\dots,k$}{
     $\pi(\kappa)=\frac{1}{N}\sum_{i=1}^{N} \pi_{\boldsymbol{x}_i}(\kappa)$ \ (\Eqref{eq:prior_cluster_prob_update})
    }
}
\end{algorithm}
DAB's concrete learning algorithm is given in Algorithm~\ref{alg:training_ig_vib}. Each epoch (outer loop in Algorithm~\ref{alg:training_ig_vib}) consists of the four alternating minimization steps presented in Section~\ref{sec:learning_algorithm}.

To render the update of $\pi$ (\Eqref{eq:prior_cluster_prob_update}) amenable to mini-batch optimization, we maintain i) a non-trainable tensor that holds the current $\pi$ ii) a moving average of the mini-batch marginals (\Eqref{eq:prior_cluster_prob_update}). 
The moving average is updated at step 4 in Algorithm~\ref{alg:training_ig_vib} such that at batch $t$ of size $B$:
\begin{align}
\pi^0({\kappa})=1/k, \ \pi^t({\kappa}) = \gamma \pi^{t-1}(\kappa)+ (1-\gamma)  \frac{1}{B}\sum_{i=1}^{B} \pi_{\boldsymbol{x}_i}(\kappa).\label{eq:momentum}
\end{align}
$0 \le \gamma \le 1$ is the momentum of the moving average. At the onset of step 4, the moving average is reset to the uniform distribution. At the end of step 4, $\pi$ is set to its moving average and is kept fixed throughout the rest of the steps, i.e., all training datapoints use the same $\pi$. 

We maintain two optimizers for the gradient updates $\nabla_{\boldsymbol{\phi}}\mathcal{L}_{\mathrm{DAB}}$ and $\nabla_{\boldsymbol{\theta}}\mathcal{L}_{\mathrm{DAB}}$. The gradient descent updates in Algorithm~\ref{alg:training_ig_vib} are written using constant learning rates $\eta_{\boldsymbol{\theta}}$, $\eta_{\boldsymbol{\phi}}$. In practice, we can use any optimizer with adaptive learning rates.  To make sure that the gradients are not propagated through $\pi_{\boldsymbol{x}}$ (\Eqref{eq:vib_ii_e_step_sol}), we apply a \texttt{tf.stop\_gradient} operator when $\mathcal{L}_{\mathrm{DAB}}$ is computed.  

In this work, we use multivariate Gaussian distributions for centroids and encoders. In this case, the centroids' parameters $\boldsymbol{\phi}$ correspond to the means and covariance matrices: $\boldsymbol{\phi} = \{\boldsymbol{\mu}_\kappa, \boldsymbol{\Sigma}_{\kappa}\}_{\kappa=1}^{k}$ and the optimal solution has a closed form~\citep{davis2006differential}. We empirically observed that using the closed-form update for the covariance matrix and gradient descent for the means facilitates optimization and speeds up convergence. To make use of the closed-form solution for the covariance matrix, we maintain non-trainable tensors holding current $\boldsymbol{\Sigma}_{\kappa}$ along with their moving averages. At the beginning of the training, the centroids's covariances are initialized to the identity matrix. The moving averages are updated in a way similar to that of $\pi$ (\Eqref{eq:momentum}). On the onset of step 3 in Algorithm~\ref{alg:training_ig_vib}, the moving average is reset to the zero matrix and is updated during the gradient updates  $\nabla_{\boldsymbol{\phi}}\mathcal{L}_{\mathrm{DAB}}$. At the end of step 3, the codebook covariances are set to their moving averages computed during this step. 

\section{VIB for Euclidean Clustering of Latent Codes}\label{sec:vib_for_clustering}
One way we can use the set of distributions $\{q_{\kappa}(\boldsymbol{z};\boldsymbol{\phi})\}_{\kappa=1}^{k}$ is to consider a mixture of $k$ distributions for the marginal $q(\boldsymbol{z};\boldsymbol{\phi})$ and trivially train it by gradient descent~\citep{alemi2018uncertainty}. To better understand the role of each $q_\kappa(\boldsymbol{z};\boldsymbol{\phi})$ during optimization, we associate a discrete random variable $\hat{Z}$ with $Z$. The value of $\hat{Z}$ indicates the assignment of $Z$ to a component $q_\kappa(\boldsymbol{z};\boldsymbol{\phi})$ of the mixture. We rewrite the upper bound of~\Eqref{eq:tau_z_centroid} in terms of $\hat{Z}$. The resulting decomposition of Proposition~\ref{proposition:vib_as_z_clustering} shows that the regularization term in the VIB (\Eqref{eq:vib_reg_decomposition}) encloses the objective of a fixed-cardinality rate-distortion function (\Eqref{eq:r_d_f_c_lag}) under some assumptions. However, computation of~\Eqref{eq:vib_reg_decomposition} requires Monte-Carlo samples of $Z$ to assign an encoder to the mixture components. The regularization terms of VIB and DAB are identical for $k=1$. The rate-distortion formulation of~\Eqref{eq:vib_reg_decomposition} motivates the DAB objective (Section~\ref{sec:ua_ib}). It also serves as a conceptual step towards the definition of a rate-distortion function acting directly on probability densities.
\begin{proposition}\label{proposition:vib_as_z_clustering} 
Let the variational marginal $q(\boldsymbol{z};\boldsymbol{\phi})$ of~\Eqref{eq:tau_z_centroid} be a mixture of $k$ distributions in $\mathbb{R}^d$ that belong to the scaled regular exponential family (def.~\ref{def:scaled_exp_family}) $\mathcal{F}^\alpha_{\psi}$ with $\alpha>0$ and log-partition function $\psi$. Let $\that{\boldsymbol{t}}_\kappa$ be the expected value of the minimal sufficient statistic $\bm{t}(Z)$ of the family when $Z \sim q_\kappa(\boldsymbol{z};\boldsymbol{\phi})$. Let $\hat{Z}$ be a (latent) categorical random variable  following distribution $q(\hat{\boldsymbol{z}})$. We assume $\hat{Z}$ is conditionally independent of $X$ given $Z$, i.e., $P(X,Y,Z)=P(X)P(Z\mid X)P(\hat{Z}\mid Z)$. The upper bound of the VIB in~\Eqref{eq:tau_z_centroid} can be decomposed as:
\begin{align}
& \mathbb{E}_{X}\big[{D}_{\mathrm{KL}}(p(\boldsymbol{z}\mid \boldsymbol{x};\boldsymbol{\theta}), q(\boldsymbol{z};\boldsymbol{\phi}))\big] = \nonumber \\
& -H(Z\mid X;\boldsymbol{\theta}) - \mathbb{E}_{X,Z}[\log f_{\psi^*_{}} (\boldsymbol{z})]+\alpha \mathbb{E}_{X, Z, \hat{Z}}[ D_{\psi^*_{}}(\boldsymbol{t}(\boldsymbol{z}),\that{\boldsymbol{t}}_{\hat{{z}}}(\boldsymbol{\phi}))] + \mathbb{E}_{X, Z}[{D}_{\mathrm{KL}}(q(\hat{\boldsymbol{z}} \mid \boldsymbol{z};\boldsymbol{\phi}),q(\hat{\boldsymbol{z}}))],
\label{eq:vib_reg_decomposition}
\end{align}
where $D_{\psi^*}$ is the Bregman divergence of  $\mathcal{F}_{\psi}$, i.e., the Bregman divergence defined by the Legendre-conjugate function $\psi^*$ of $\psi$. $f_{\psi^*}$ is a non-negative function that does not depend on the natural parameter $\boldsymbol{\phi}$.
\end{proposition}

\begin{proof}
%By marginalization, $\tau(\boldsymbol{z})$ can be written as:
% \begin{align}
% \nonumber
% \tau(\boldsymbol{z}) &=   \sum_{\hat{\boldsymbol{z}}\in \mathcal{\hat{Z}}_s} \tau(\boldsymbol{z},\hat{\boldsymbol{z}}) \\  \nonumber &= \sum_{\kappa=1}^{k} \tau(\boldsymbol{z},\boldsymbol{\mu}_\kappa) \\  \nonumber 
% &= \sum_{\kappa=1}^{k} p(\hat{Z}=\boldsymbol{\mu}_\kappa) \tau(\boldsymbol{z}\mid \hat{Z}=\boldsymbol{\mu}_\kappa) \\  \nonumber
% &= \sum_{\kappa=1}^{k} \pi(\kappa) \tau_{\kappa}(\boldsymbol{z}).
% \end{align}
We expand the upper bound in~\Eqref{eq:tau_z_centroid}:
\begin{align}
&\mathbb{E}_{X}\big[{D}_{\mathrm{KL}}(p(\boldsymbol{z}\mid \boldsymbol{x};\boldsymbol{\theta}), q(\boldsymbol{z};\boldsymbol{\phi}))\big] = \nonumber \\ 
&\int p(\boldsymbol{x}) p(\boldsymbol{z}\mid\boldsymbol{x};\boldsymbol{\theta})\log p(\boldsymbol{z}\mid \boldsymbol{x};\boldsymbol{\theta})d\boldsymbol{z}d\boldsymbol{x}-\int p(\boldsymbol{x}) p(\boldsymbol{z}\mid\boldsymbol{x};\boldsymbol{\theta})\log q(\boldsymbol{z};\boldsymbol{\phi})d\boldsymbol{z}d\boldsymbol{x}.
\label{eq:vib_reg}
\end{align}
The first term of~\Eqref{eq:vib_reg} is the negative conditional differential entropy of the encoder, i.e., $-H(Z\mid X;\boldsymbol{\theta})$. We will focus on the second term of~\Eqref{eq:vib_reg}. For a fixed $\boldsymbol{z}$:
\begin{align}
\log q(\boldsymbol{z};\boldsymbol{\phi}) & = \mathbb{E}_{\hat{Z} \mid \boldsymbol{z} } [\log q(\boldsymbol{z};\boldsymbol{\phi})] \nonumber \\
&= \mathbb{E}_{\hat{Z} \mid \boldsymbol{z}   } [\log q(\boldsymbol{z};\boldsymbol{\phi}) + \log q(\hat{\boldsymbol{z}} \mid \boldsymbol{z};\boldsymbol{\phi})-\log q(\hat{\boldsymbol{z}} \mid \boldsymbol{z};\boldsymbol{\phi})] \nonumber \\
&= \mathbb{E}_{\hat{Z} \mid \boldsymbol{z}  } [\log q(\boldsymbol{z}, \hat{\boldsymbol{z}};\boldsymbol{\phi})-\log q(\hat{\boldsymbol{z}} \mid \boldsymbol{z};\boldsymbol{\phi})]
\nonumber \\
%&= \mathbb{E}_{ p(\hat{\boldsymbol{z}} \mid \boldsymbol{z})} [\log \tau(\boldsymbol{z}, \hat{\boldsymbol{z}})] + H(\hat{Z}\mid \boldsymbol{z}) \nonumber \\
&= \mathbb{E}_{\hat{Z} \mid \boldsymbol{z} } [\log q(\boldsymbol{z} \mid \hat{\boldsymbol{z}};\boldsymbol{\phi}) + \log q(\hat{\boldsymbol{z}})-\log q(\hat{\boldsymbol{z}} \mid \boldsymbol{z};\boldsymbol{\phi})].
\label{eq:log_tau}
\end{align}
%$\hat{\boldsymbol{z}} \in \{\boldsymbol{\mu}_\kappa\}_{\kappa=1}^{k}$ where $\boldsymbol{\mu}_k$ is the expected value of $t(Z)$ when $Z \sim \tau_\kappa(\boldsymbol{z})$. 
We first analyze the first term in ~\Eqref{eq:log_tau}. By definition of $\hat{Z}$, $q(\boldsymbol{z} \mid \hat{\boldsymbol{z}};\boldsymbol{\phi})=q_{\hat{\boldsymbol{z}}}(\boldsymbol{z};\boldsymbol{\phi})$. Let $\that{\boldsymbol{t}}_\kappa(\boldsymbol{\phi})$ be the expected value of $\boldsymbol{t}(Z)$ when $Z$ is sampled from the $\kappa$--th component of the mixture: $Z \sim q_\kappa(\boldsymbol{z};\boldsymbol{\phi})$. Since $q_\kappa(\boldsymbol{z};\boldsymbol{\phi})$ belongs to the regular exponential family, its Bregman form (\Eqref{eq:exp_bregman_forma}) is:
\begin{align}
q_\kappa(\boldsymbol{z};\boldsymbol{\phi}) = \exp\left(-D_{{\psi^*_{}}}(\boldsymbol{t}(\boldsymbol{z}),\that{\boldsymbol{t}}_\kappa(\boldsymbol{\phi}))\right)f_{\psi^*_{}}(\boldsymbol{z}),
\end{align}
where ${\psi^*_{}}$ is the conjugate of the log-partition function $\psi$ of the family, $D_{\psi^*_{}}$ is the Bregman divergence defined by ${\psi^*_{}}$, and $f_{\psi^*_{}}$ given in~\Eqref{eq:exp_bregman_formb}.  %For example, for a $d$-diagonal Gaussian $\mathcal{N}(\mu,\sigma^2I)$, $D_{\psi^*}$ corresponds to the Euclidean norm scaled by $2\sigma^2$, and $f_{\psi^*}(\boldsymbol{z})=1/\sqrt{2\pi\sigma^2}^{d}$ ($f_{\psi^*}$ is a constant in this case). For the derivation of this example, please read example 9 in \cite{banerjee2005clustering}. %We will assume that the covariance matrix is kept fixed (and potentially initialized by random features \cite{rahimi2007random}; a similar formation of the covariance matrix is used in \cite{liu2020simple}) so that $D_\phi$ is not learnable. 
In general, we can consider a scaled exponential family with Bregman form (see def.~\ref{def:scaled_exp_family}):
\begin{align}
q_\kappa(\boldsymbol{z};\boldsymbol{\phi}) = \exp\left(-\alpha D_{{\psi^*_{}}}(\boldsymbol{t}(\boldsymbol{z}),\boldsymbol{\hat{{t}}}_\kappa(\boldsymbol{\phi}))\right)f_{\alpha \psi^*_{}}(\boldsymbol{z}), \ \alpha>0.
\label{eq:scaled_exp_bregman}
\end{align}
We now look at the last two terms of ~\Eqref{eq:log_tau}:
\begin{align}
\mathbb{E}_{\hat{Z} \mid \boldsymbol{z} } [\log q(\hat{\boldsymbol{z}})-\log q(\hat{\boldsymbol{z}} \mid \boldsymbol{z};\boldsymbol{\phi})] = - {D}_{\mathrm{KL}}(q(\hat{\boldsymbol{z}} \mid \boldsymbol{z};\boldsymbol{\phi}),q(\hat{\boldsymbol{z}})).
\label{eq:kl_decomp}
\end{align}
By taking expectation of~\Eqref{eq:log_tau} with respect to $p(\boldsymbol{x}), p(\boldsymbol{z}\mid \boldsymbol{x};\boldsymbol{\theta})$ and using~\Eqref{eq:scaled_exp_bregman} and~\Eqref{eq:kl_decomp}, we can rewrite~\Eqref{eq:vib_reg}:
\begin{align}
&\mathbb{E}_{X}\big[{D}_{\mathrm{KL}}(p(\boldsymbol{z}\mid \boldsymbol{x};\boldsymbol{\theta}), q(\boldsymbol{z};\boldsymbol{\phi}))\big] = \nonumber \\
&-H(Z\mid X;\boldsymbol{\theta}) - \mathbb{E}_{X,Z}[\log f_{\psi^*_{}} (\boldsymbol{z})]+\alpha \mathbb{E}_{X, Z, \hat{Z}}\big[ D_{\psi^*_{}}(\boldsymbol{t}(\boldsymbol{z}),\that{\boldsymbol{t}}_{\hat{{z}}}(\boldsymbol{\phi}))\big] + \mathbb{E}_{X, Z}\big[{D}_{\mathrm{KL}}(q(\hat{\boldsymbol{z}} \mid \boldsymbol{z};\boldsymbol{\phi}),q(\hat{\boldsymbol{z}}))\big].
\label{eq:decomp}
\end{align}

\end{proof}

\ben{BE notes: Trying to relate this to Eq.~\ref{eq:ua_ib}.}
\ben{\begin{align}
    E_{P_X, {Q}}[D(p_x, {q_\kappa})]
    &= E_{P_X, {Q}}[KL(p_x \| {q_\kappa})] \\
    &= E_{p(x), \pi_x}[KL(p(z \mid x) \| {q}_\kappa(z))] \\
    &= E_{p(x), p(\kappa \mid x)}[KL(p(z \mid x) \| q_\kappa(z))] \\
    &= E_{p(x), p(\kappa \mid x)}[E_{p(z \mid x)}[\log (p(z \mid x) / q_\kappa(z))] \\
    &= E_{p(x), p(z \mid x)}[\log p(z \mid x)] - E_{p(x)p(z \mid x), p(\kappa \mid x)}[\log q_\kappa(z)].
\end{align}}
\ben{Explanation of the math above:
\begin{enumerate}
  \setcounter{enumi}{10}
    \item assuming $D(\cdot, \cdot)$ is forward KL.
    \item Sampling $p_x \sim P_X$ is equivalent to sampling $\{p(z \mid x) \mid x \sim p(x)\}$
\end{enumerate}
We can compare this to a standard VIB with a mixture model prior. Define $q(z) \triangleq \frac{1}{k}\sum_\kappa q_\kappa(z)$.
\begin{align}
    E_{p(x)}[KL(p(z \mid x) \| q(\tau)]
    &= E_{p(x)}[E_{p(z \mid x)}[\log p(z \mid x)  - q(z)]] \\
        &= E_{p(x), p(z \mid x)}[\log p(z \mid x)] - E_{p(x), p(z \mid x)}[\log q(z)].
\end{align}}
\ben{
This looks similar to the equation above, except that the second term regularizes the latent code $z$ to be similar to the \emph{$x$-specific cluster} $q_\kappa(z)$, rather than to the full mixture $q(z)$.}

\ben{However, we can further connect the second term in both equations using Jensen's inequality. We'll now assume that the mixture weights $p(\kappa \mid x)$ are learned:
\begin{align}
    \log q(z) &= \log E_{p(\kappa \mid x)}[q_\kappa(z)] \\
    &\ge  E_{p(\kappa \mid x)}[\log q_\kappa(z)].
\end{align}
Substituting this into the above, we get
\begin{align}
    E_{p(x)}[KL(p(z \mid x) \| q(z)]
    &= E_{p(x)}[E_{p(z \mid x)}[\log p(z \mid x)  - q(z)]] \\
    &\le E_{p(x), p(z \mid x)}[\log p(z \mid x)] - E_{p(x), p(z \mid x),p(\kappa \mid x)}[\log q_\kappa(z)].
\end{align}
The inequality above is the opposite direction because it is subtracted.
This last equation is the same as the UA IB. Thus, we can conclude that the IB with a mixture model prior (with learned mixture weights) is closely related to the UA IB, with one being a lower bound on the other.}

When minimizing~\Eqref{eq:decomp} with respect to $\boldsymbol{\theta}$, the Bregman term encourages encoder $p(\boldsymbol{z} \mid \boldsymbol{x};\boldsymbol{\theta})$ that generates samples $\boldsymbol{z}$ whose sufficient statistics are close to one of the means $\that{\boldsymbol{t}}_\kappa$ in terms of $D_\psi^*$. This term, in turn, encourages:

\circled{1} encoders that collapse to a single atom $\that{\boldsymbol{t}}_\kappa$: $q(\kappa\mid \boldsymbol{z};\boldsymbol{\phi})\rightsquigarrow1$. This is counterbalanced by the KL term of~\Eqref{eq:decomp}.

\circled{2} low-entropy encoders that generate almost deterministic sufficient statistics for its samples: $\boldsymbol{t}(\boldsymbol{z}) \rightsquigarrow \that{\boldsymbol{t}}_{\kappa}$. The negative entropy term in~\Eqref{eq:decomp} helps avoid such degenerate solutions. A similar observation for the special case of a single Gaussian $q(\boldsymbol{z};\boldsymbol{\phi})$ (note the KL term in~\Eqref{eq:decomp} vanishes and $\hat{Z}$ can be dropped in the second expectation in this case) following, however, an entirely algebraic route, is also made by~\citet{hoffman2017beta}. Here, we present an information-theoretic perspective of this trade-off.

% By minimizing~\Eqref{eq:vib_reg_decomposition} with respect to $\boldsymbol{\theta}$, encoder $p(\boldsymbol{z} \mid \boldsymbol{x};\boldsymbol{\theta})$ is regularized to generate samples $\boldsymbol{z}$ whose sufficient statistics are close to the means $\that{\boldsymbol{t}}_\kappa$ in terms of $D_\psi^*$. For Gaussian $q_\kappa(\boldsymbol{z};\boldsymbol{\phi})$, minimizing~\Eqref{eq:vib_reg_decomposition} would encourage small covariance for the encoder and mixture's components. This is counterbalanced by the negative entropy term in~\Eqref{eq:vib_reg_decomposition}, without which the encoder would collapse to a single atom, i.e., to one of the expected values of $q_\kappa(\boldsymbol{z};\boldsymbol{\phi})$. Maximizing its entropy by the first term in~\Eqref{eq:vib_reg_decomposition} avoids such degenerate solutions. The scale $\alpha$ balances the trade-off between the distortion level and the stochasticity of $Z$. A similar observation following, however, an entirely algebraic route, is also made by~\citet{hoffman2017beta}. Here, we present an information-theoretic perspective of this trade-off.
Keeping everything but $\boldsymbol{\phi}$ fixed, minimizing~\Eqref{eq:vib_reg} over a finite number of sampled latent codes $\boldsymbol{z}$ is equivalent to MLE with a mixture distribution. In the case of a Gaussian mixture, this is equivalent to soft K-means clustering in the latent space in $\mathbb{R}^d$. For distributions for which $\boldsymbol{t}(\boldsymbol{z})=\boldsymbol{z}$\footnote{For example, for Gaussian $\mathcal{N}(\boldsymbol{\mu},\boldsymbol{\Sigma})$ with constant and a priori known covariance matrix $\boldsymbol{\Sigma}$, $\boldsymbol{t}(\boldsymbol{z})=\boldsymbol{z}$, $\boldsymbol{\phi}=\boldsymbol{\mu}=\that{\boldsymbol{t}}$.}, minimizing~\Eqref{eq:vib_reg_decomposition} with respect to $\boldsymbol{\phi}$ amounts to computing the Rate Distortion Finite Cardinality (RDFC) function (\Eqref{eq:r_d_f_c_lag}) with the Bregman distortion $D_{\psi^*}$~\citep{banerjee2004information}. The support $\mathcal{\hat{Z}}$ of $\hat{Z}$ to be learned has cardinality $k$ and corresponds to the sufficient statistic means $\mathcal{\hat{Z}}=\{\that{\boldsymbol{t}}_\kappa\}_{\kappa=1}^{k}$\footnote{We implicitly redefine $\hat{Z}$ to take values in $\mathcal{\hat{Z}}=\{\that{\boldsymbol{t}}_\kappa\}_{\kappa=1}^{k}$.}. In our case, the log-likelihood of latent codes sampled by the encoder is maximized instead. Moreover, the source (encoder) is not apriori known but its parameters $\boldsymbol{\theta}$ are trainable during optimization. Using the decomposition of~\Eqref{eq:decomp}, the first two terms can be ignored since $H(Z\mid X;\boldsymbol{\theta})$ and $\log f_{\psi^*}(\boldsymbol{z})$ do not depend on $\boldsymbol{\phi}$.

\section{Additional Experiments}
%%%%%%%%%%%%%%%%%%%%
%%%%%%%%%%%%%%%%%%%%
%\clearpage
\subsection{Ablation Studies on CIFAR-10 \label{sec:ablation}}

In Table~\ref{table:OoD_scores}, we compare the OOD performance of DAB models when using other commonly-used OOD metrics. As expected, the proposed distortion score, that is \textit{explicitly} minimized for the training datapoints via the loss function in~\Eqref{eq:ua_ib}, yields better OOD detection performance.  
%%%%%%%%%%%%%
\begin{table}[!ht]
\centering
\caption{\textbf{DAB performance with alternative OOD scores. }$D_\KL$ refers to the Kullback-Leibler distortion of~\Eqref{eq:distance_3}. $H$ refers to the entropy of the decoder's classifier: $H \triangleq \mathbb{E}_{Y,Z\mid \boldsymbol{x}}[- \log m(\boldsymbol{y}\mid \boldsymbol{z};\boldsymbol{\theta})]$. Finally, $p_{max}$ refers to the maximum probability of the classifier: $p_{\max}\triangleq \argmax_{c}\mathbb{E}_{Z\mid \boldsymbol{x}}[ m(Y=c \mid \boldsymbol{z};\boldsymbol{\theta})]$. $p_{\max}$ and $H$ are approximated by Monte Carlo with a single sample of $Z$. The Kullback-Leibler divergence from the learned centroids is more sensitive to input variations rendering the distortion of~\Eqref{eq:distance_3} a better indicator of an OOD input. Moreover, it is Monte Carlo sample-free for Gaussian encoders and centroids.}\label{table:OoD_scores}
\begin{tabular}{cc|c|c|c|}
        \hline
\multicolumn{1}{|c|}{\multirow{2}{*}{\textbf{OOD score}}}    & \multicolumn{2}{c|}{\textbf{SVHN}}    & \multicolumn{2}{c|}{\textbf{CIFAR-100}}                     \\ \cline{2-5}
\multicolumn{1}{|c|}{}          & \textbf{AUROC} $\uparrow$  & \textbf{AUPRC}   $\uparrow$      & \textbf{AUROC} $\uparrow$ & \textbf{AUPRC}         $\uparrow$                                                                \\ \hline

\multicolumn{1}{|c|}{\begin{tabular}[c]{@{}c@{}} $D_{\KL}$ \end{tabular}} & $\mathbf{0.986 \pm 0.004}$          & $\bm{0.994 \pm 0.002}$               & $\mathbf{0.922 \pm 0.002}$             & \bm{${0.915 \pm 0.002}$}                                                                                                          \\ \hline

\multicolumn{1}{|c|}{\begin{tabular}[c]{@{}c@{}} $H$ \end{tabular}} & $0.964 \pm 0.009$          & $0.982 \pm 0.005$               & $0.891  \pm 0.003$ & $0.883 \pm 0.003$                                                                                                          \\ \hline
\multicolumn{1}{|c|}{\begin{tabular}[c]{@{}c@{}} $1-p_{max}$  \end{tabular}} & $  0.959 \pm 0.009   $        & $0.978 \pm 0.006$               & $0.889 \pm 0.003$           &  $0.875 \pm 0.003$    \\        \hline
\end{tabular}
\end{table}
%%%%%%%%%%%%%

In the rest of this section, we study the effect of the DAB hyperparameters, also listed in Table~\ref{table:regression_hyperparameters}, on the OOD performance of our model.

In Table~\ref{tab:ood_vary_k}, we do an ablation study on the RDFC cardinality $k$. We see that a larger number of centroids improves the quality of the uncertainty estimates. However, further increasing the codebook size with $k>10$ yields diminishing performance benefits. Similar to~\Figref{fig:10_clusters}, we sought to justify this model's behavior via visual inspection of the codebook. We noticed that when $k>10$ some centroids are assigned to only a small number of training datapoints. This observation can serve as a recipe for choosing the codebook size: albeit a larger codebook will not harm performance, unutilized entries indicate that a smaller codebook can achieve similar quality for the model's uncertainty estimates.
%%%%%%%%%%%%%%%%%%%%
\begin{table}[!ht]
\centering
\caption{\textbf{Ablation study over codebook size $\boldsymbol{k}$.} A single Gaussian code $q(\boldsymbol{z})$ does not discriminate well CIFAR-10 from the visually similar datapoints of CIFAR-100. As we increase the number of centroids, DAB progressively becomes better at distinguishing these datasets. DAB reaches competitive performance with a small number of 10 centroids. The performance remains roughly the same when using a larger cardinality $k>10$.}
\label{tab:ood_vary_k}
\begin{tabular}{cc|c|c|c|}
        \hline
\multicolumn{1}{|c|}{\multirow{2}{*}{\textbf{}}}    & \multicolumn{2}{c|}{\textbf{SVHN}}    & \multicolumn{2}{c|}{\textbf{CIFAR-100}}                     \\ \cline{2-5}
\multicolumn{1}{|c|}{}          & \textbf{AUROC} $\uparrow$  & \textbf{AUPRC}   $\uparrow$      & \textbf{AUROC} $\uparrow$ & \textbf{AUPRC}         $\uparrow$                                                                \\ \hline

\multicolumn{1}{|c|}{\begin{tabular}[c]{@{}c@{}} $k=10$\end{tabular}} & $\mathbf{0.986 \pm 0.004}$          & $\mathbf{0.994 \pm 0.002}$               & $\mathbf{0.922 \pm 0.002}$             & $\mathbf{0.915 \pm 0.002}$                                                                                                          \\ \hline
\multicolumn{1}{|c|}{\begin{tabular}[c]{@{}c@{}} $k=5$\end{tabular}} & $0.968 \pm 0.031$               & $0.986 \pm 0.012$             & $0.912 \pm 0.009$ &  $0.907 \pm 0.007$                                                                                                         \\ \hline
\multicolumn{1}{|c|}{\begin{tabular}[c]{@{}c@{}} $k=1$ \\ vanilla VIB~\citep{alemi2016deep}\end{tabular}} & $0.906 \pm 0.052$          & $0.958 \pm 0.026$    & $0.746 \pm 0.023$           & $0.764 \pm 0.026$           \\ \hline
\end{tabular}
\end{table}
%\clearpage 

In Table~\ref{tab:ood_vary_alpha}, we study the effect of the temperature $\alpha$ (\Eqref{eq:vib_ii_e_step_sol}). We verify that $\alpha$ controls the strength of the statistical distance when comparing a datapoint with the codebook. For small values of $\alpha$, the model exhibits a uniformity-tolerance for the datapoints that lie well beyond the support of the training dataset. On the other hand, the distribution $\pi_{\boldsymbol{x}}$ (\Eqref{eq:vib_ii_e_step_sol}) becomes sharper for larger values of $\alpha$. A sharper distribution translates to a more informative centroid assignment for datapoint $\boldsymbol{x}$. Subsequently, an informative codebook helps the model to successfully mark the areas of the input distribution that is familiar with.

\begin{table}[!ht]
\centering
\caption{\textbf{Ablation study over temperature $\boldsymbol{\alpha}$. } With small values of $\alpha$, the model fails to discriminate inputs successfully, which it should be less confident about. Large values of $\alpha$ lead to a more concentrated assignment of the training datapoints to the centroids. This, in turn, provides the model with more effective OOD scores that sufficiently penalize large distances from the codebook.}
\label{tab:ood_vary_alpha}
\begin{tabular}{cc|c|c|c|}
        \hline
\multicolumn{1}{|c|}{\multirow{2}{*}{\textbf{}}}    & \multicolumn{2}{c|}{\textbf{SVHN}}    & \multicolumn{2}{c|}{\textbf{CIFAR-100}}                     \\ \cline{2-5}
\multicolumn{1}{|c|}{}          & \textbf{AUROC} $\uparrow$  & \textbf{AUPRC}   $\uparrow$      & \textbf{AUROC} $\uparrow$ & \textbf{AUPRC}         $\uparrow$                                                                \\ \hline

\multicolumn{1}{|c|}{\begin{tabular}[c]{@{}c@{}} $\alpha=\phantom{0}0.1$\end{tabular}} & $ 0.932\pm 0.038$ & $ 0.972\pm 0.018$ & $ 0.756\pm 0.031$         & $0.776\pm 0.032$                                                                                                          \\ \hline
\multicolumn{1}{|c|}{\begin{tabular}[c]{@{}c@{}} $\alpha=\phantom{0}0.5$\end{tabular}} & $ 0.958\pm 0.045$               & $ 0.982 \pm 0.019$             & $ 0.878\pm 0.057$ &  $ 0.879\pm 0.043$                                                                                                         \\ \hline
\multicolumn{1}{|c|}{\begin{tabular}[c]{@{}c@{}} $\alpha=\phantom{0}1.0$\end{tabular}} & $\boldsymbol{0.986 \pm 0.004}$          & $\boldsymbol{0.994 \pm 0.002}$               & $\boldsymbol{0.922 \pm 0.002}$             & $\boldsymbol{0.915 \pm 0.002}$            \\ \hline
\multicolumn{1}{|c|}{\begin{tabular}[c]{@{}c@{}} $\alpha=\phantom{0}2.0$ \end{tabular}} & $ 0.989\pm 0.003$ & $ 0.995 \pm 0.001$    & $ 0.924\pm 0.001$   & $ 0.918\pm 0.002$           \\ \hline
\multicolumn{1}{|c|}{\begin{tabular}[c]{@{}c@{}} $\alpha=10.0$\end{tabular}}          & $ 0.982\pm 0.005$    & $ 0.991\pm 0.002$        & $ 0.923\pm 0.002$    & $ 0.916\pm 0.002$           \\ \hline
\end{tabular}
\end{table}
%%%%%%%%%%%%%%%%%%%%%%%%%%%%%%%%%%%
In Table~\ref{tab:ood_vary_beta}, we vary the regularization coefficient $\beta$~(\Eqref{eq:ua_ib}). We see that the model achieves the best performance within a range of $\beta$. For smaller values of $\beta$, the distortion term in~\Eqref{eq:ua_ib} is disregarded. Therefore, the main network is not restricted to producing encoders that can be well-represented by the codebook. For larger values of $\beta$, the training datapoints get closely attached to the centroids. This results in statistical balls of small radius (\Figref{d_train_enc_center}) effectively leaving out novel, in-distribution datapoints.

\begin{table}[!ht]
\centering
\caption{\textbf{Ablation study over regularization coefficient $\boldsymbol{\beta}$}. The model is best performing within a range of values. Large values of $\beta$ correspond to small balls around the centroids (\Figref{d_train_enc_center}) and vice-versa. The balls should be small enough to exclude OOD inputs but large enough to include unseen, in-distribution points to which the model can generalize.}
\label{tab:ood_vary_beta}
\begin{tabular}{cc|c|c|c|}
        \hline
\multicolumn{1}{|c|}{\multirow{2}{*}{\textbf{}}}    & \multicolumn{2}{c|}{\textbf{SVHN}}    & \multicolumn{2}{c|}{\textbf{CIFAR-100}}                     \\ \cline{2-5}
\multicolumn{1}{|c|}{}          & \textbf{AUROC} $\uparrow$  & \textbf{AUPRC}   $\uparrow$      & \textbf{AUROC} $\uparrow$ & \textbf{AUPRC}         $\uparrow$                                                                \\ \hline

\multicolumn{1}{|c|}{\begin{tabular}[c]{@{}c@{}} $\beta=0.0001$\end{tabular}} & $ 0.925\pm 0.429$ & $0.965 \pm 0.02$ & $ 0.70\phantom{0}\pm 0.019$         & $ 0.697\pm 0.02\phantom{0}$                                                                                                          \\ \hline
\multicolumn{1}{|c|}{\begin{tabular}[c]{@{}c@{}} $\beta=0.0005$\end{tabular}} & $ 0.98\phantom{0}\pm 0.009$               & $ 0.99\phantom{0}\pm 0.005$             & $0.917 \pm 0.002$ &  $ 0.91\phantom{0}\pm 0.003$                                                                                                         \\ \hline
\multicolumn{1}{|c|}{\begin{tabular}[c]{@{}c@{}} $\beta=0.001\phantom{0}$\end{tabular}} & $\boldsymbol{0.986 \pm 0.004}$          & $\boldsymbol{0.994 \pm 0.002}$               & $\boldsymbol{0.922 \pm 0.002}$             & $\boldsymbol{0.915 \pm 0.002}$            \\ \hline
\multicolumn{1}{|c|}{\begin{tabular}[c]{@{}c@{}} $\beta=0.005\phantom{0}$ \end{tabular}} & $ 0.985\pm 0.004$ & $ 0.993\pm 0.002$    & $ 0.921\pm 0.002$   & $ 0.914\pm 0.002$           \\ \hline
\multicolumn{1}{|c|}{\begin{tabular}[c]{@{}c@{}} $\beta=0.01\phantom{00}$\end{tabular}}          & $ 0.977\pm 0.01\phantom{0}$    & $ 0.988\pm 0.005$        & $ 0.914\pm 0.002$    & $0.907\pm 0.001$           \\ \hline
\end{tabular}
\end{table}
In Table~\ref{tab:ood_vary_z}, we are sweeping the bottleneck dimension. In Table~\ref{table:OoD_params_accuracy}, we see that 8-dimensional latent features can capture the information needed for the CIFAR-10 classification task. Further increasing the bottleneck size leads to irrelevant features that have no effect. On the other hand, smaller features disregard essential aspects of the input.

\begin{table}[!ht]
\centering
\caption{\textbf{Ablation study over bottleneck dimension. }Larger latent features improve OOD capability until a performance plateau is reached.}
\label{tab:ood_vary_z}
\begin{tabular}{cc|c|c|c|}
        \hline
\multicolumn{1}{|c|}{\multirow{2}{*}{\textbf{}}}    & \multicolumn{2}{c|}{\textbf{SVHN}}    & \multicolumn{2}{c|}{\textbf{CIFAR-100}}                     \\ \cline{2-5}
\multicolumn{1}{|c|}{}          & \textbf{AUROC} $\uparrow$  & \textbf{AUPRC}   $\uparrow$      & \textbf{AUROC} $\uparrow$ & \textbf{AUPRC}         $\uparrow$                                                                \\ \hline

\multicolumn{1}{|c|}{\begin{tabular}[c]{@{}c@{}} $\mathrm{dim}(\boldsymbol{z})=2\phantom{0}$\end{tabular}} & $ 0.748\pm 0.03\phantom{0}$ & $ 0.797\pm 0.014$ & $ 0.678\pm 0.014$         & $ 0.59\phantom{0}\pm 0.008$                                                                                                          \\ \hline
\multicolumn{1}{|c|}{\begin{tabular}[c]{@{}c@{}} $\mathrm{dim}(\boldsymbol{z})=4\phantom{0}$\end{tabular}} & $ 0.974\pm 0.01\phantom{0}$  & $ 0.98\phantom{0}\pm 0.004$ & $ 0.877\pm 0.012$ &  $ 0.872\pm 0.008$                                                                                                         \\ \hline
\multicolumn{1}{|c|}{\begin{tabular}[c]{@{}c@{}} $\mathrm{dim}(\boldsymbol{z})=8\phantom{0}$\end{tabular}} & $\boldsymbol{0.986 \pm 0.004}$          & $\boldsymbol{0.994 \pm 0.002}$               & $\boldsymbol{0.922 \pm 0.002}$             & $\boldsymbol{0.915 \pm 0.002}$            \\ \hline
\multicolumn{1}{|c|}{\begin{tabular}[c]{@{}c@{}} $\mathrm{dim}(\boldsymbol{z})=10$ \end{tabular}} & $ 0.983\pm 0.005$ & $ 0.991\pm 0.003$    & $ 0.924\pm 0.002$   & $0.915 \pm 0.001$           \\ \hline
\end{tabular}
\end{table}
%%%%%%%%%%%
Finally, the model was not sensitive to typical values, i.e. $>0.9$, for the momentum $\gamma$.
%%%%%%%%%%%%%%%%%%%%%%%%%%%%%%%%%%%%%%
\subsection{DAB for detecting CIFAR-10 with noise corruptions\label{sec:corruption}}
 Table~\ref{fig:dab_on_corrupted}  shows the AUROC scores for the DAB of Section~\ref{sec:ood_experiments} on test CIFAR-10 versus test CIFAR-10 with common noise corruptions~\citep{hendrycks2018benchmarking}.
%\clearpage
\begin{figure}
\centering
\begin{subfigure}{0.333\textwidth}
  \centering  \includegraphics[width=.9\linewidth]{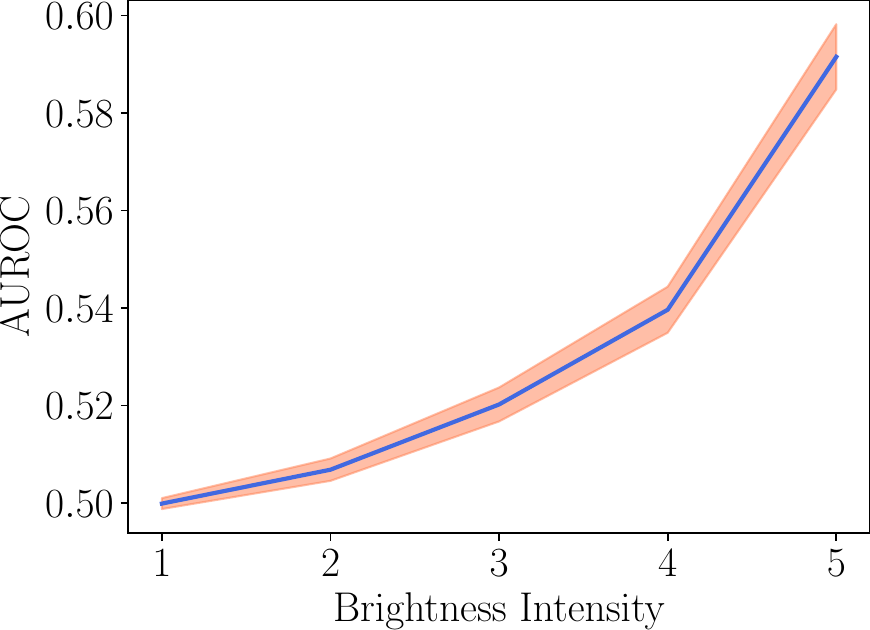}
  \caption{}
\end{subfigure}%
\begin{subfigure}{0.333\textwidth}
  \centering
  \includegraphics[width=.9\linewidth]{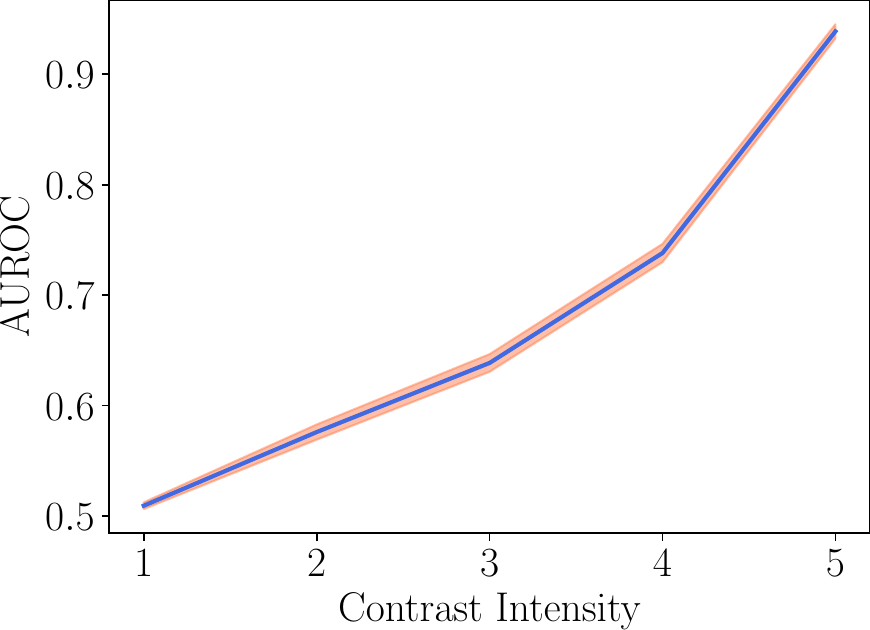}
  \caption{}
\end{subfigure}%\par\medskip
\begin{subfigure}{0.333\textwidth}
  \centering
  \includegraphics[width=.9\linewidth]{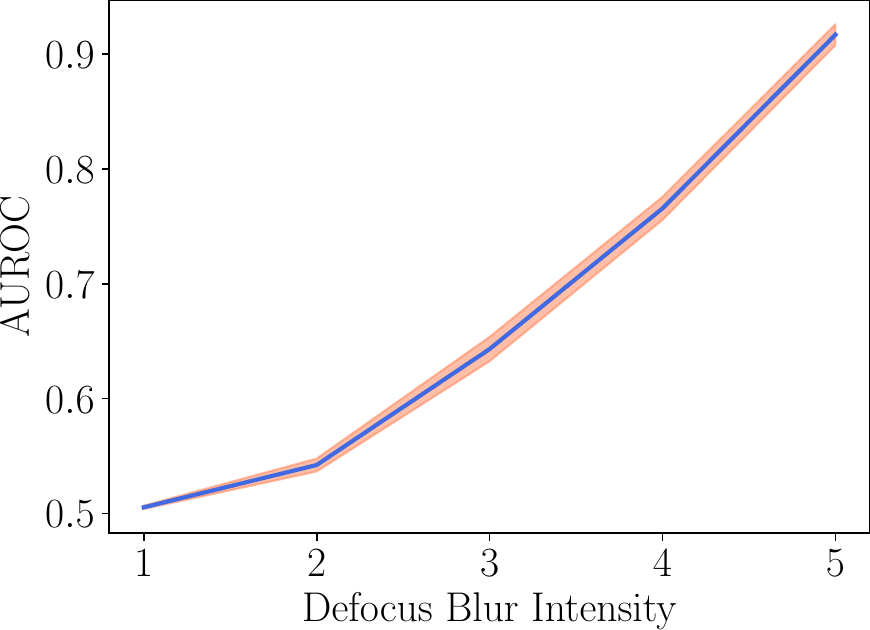}
  \caption{}
\end{subfigure}
%%%%%%%%%%%%%%%%%%%%%
\begin{subfigure}{0.333\textwidth}
  \centering  \includegraphics[width=.9\linewidth]{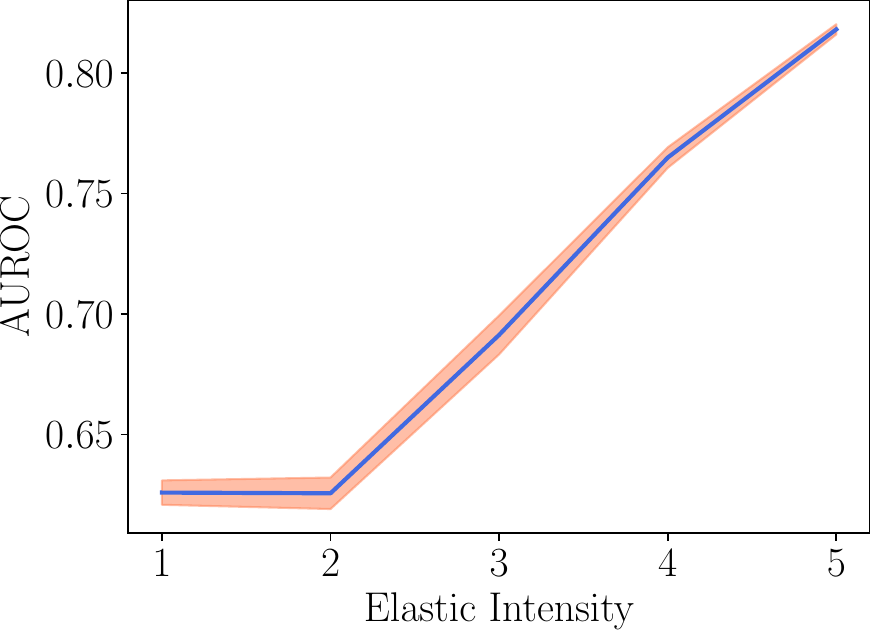}
  \caption{}
\end{subfigure}%
\begin{subfigure}{0.333\textwidth}
  \centering
  \includegraphics[width=.9\linewidth]{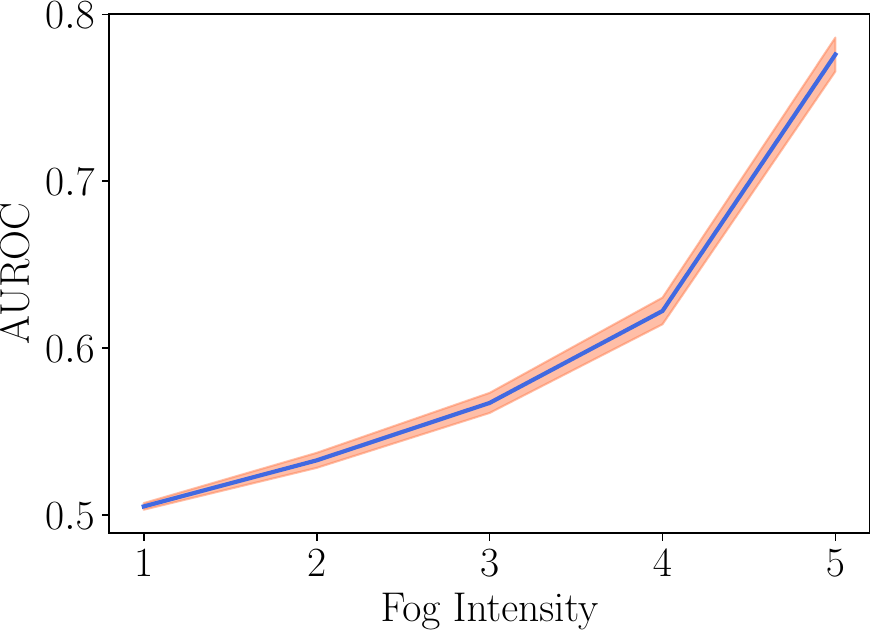}
  \caption{}
\end{subfigure}%\par\medskip
\begin{subfigure}{0.333\textwidth}
  \centering
  \includegraphics[width=.9\linewidth]{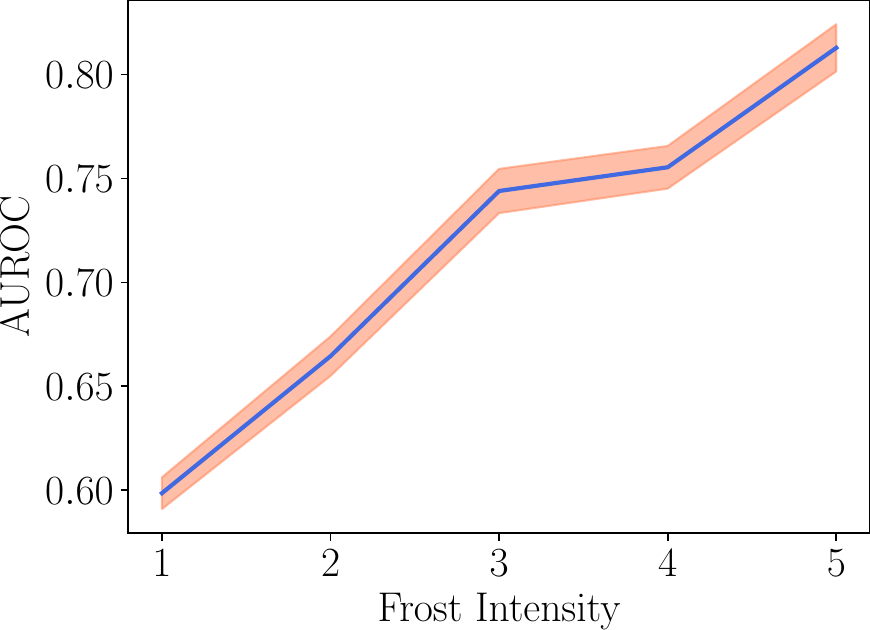}
  \caption{}
\end{subfigure}
%%%%%%%%%%%%%%%%%%%%%
\begin{subfigure}{0.333\textwidth}
  \centering  \includegraphics[width=.9\linewidth]{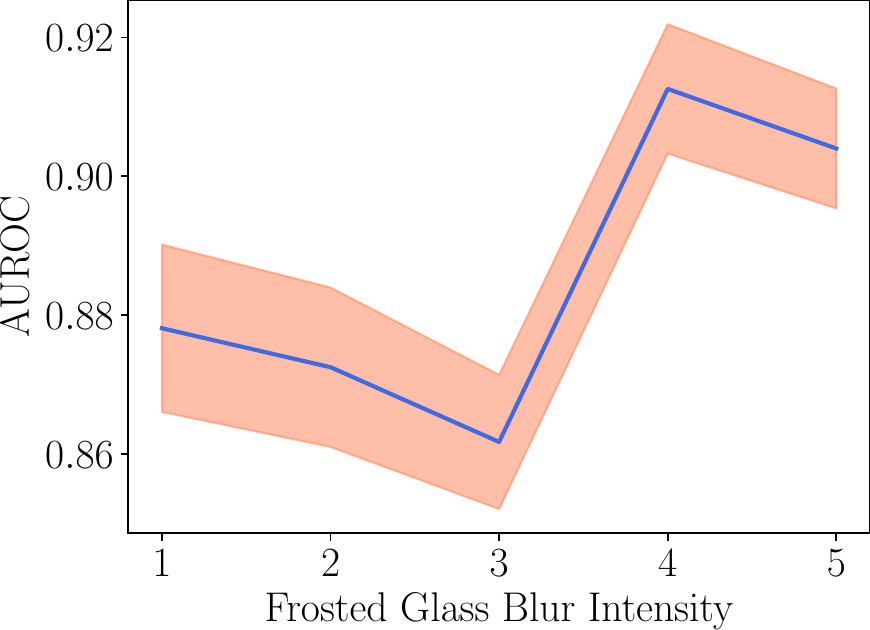}
  \caption{}
\end{subfigure}%
\begin{subfigure}{0.333\textwidth}
  \centering
  \includegraphics[width=.9\linewidth]{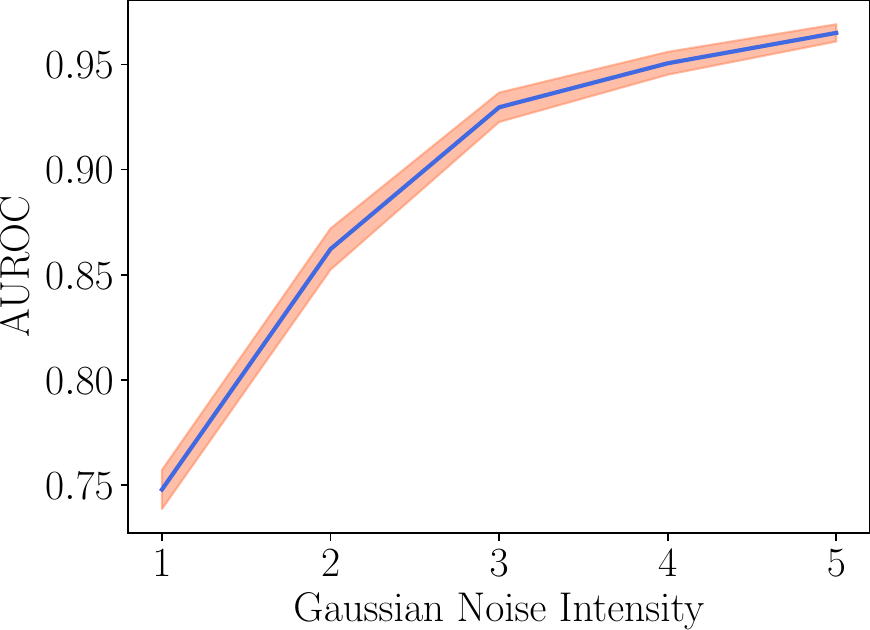}
  \caption{}
\end{subfigure}%\par\medskip
\begin{subfigure}{0.333\textwidth}
  \centering
  \includegraphics[width=.9\linewidth]{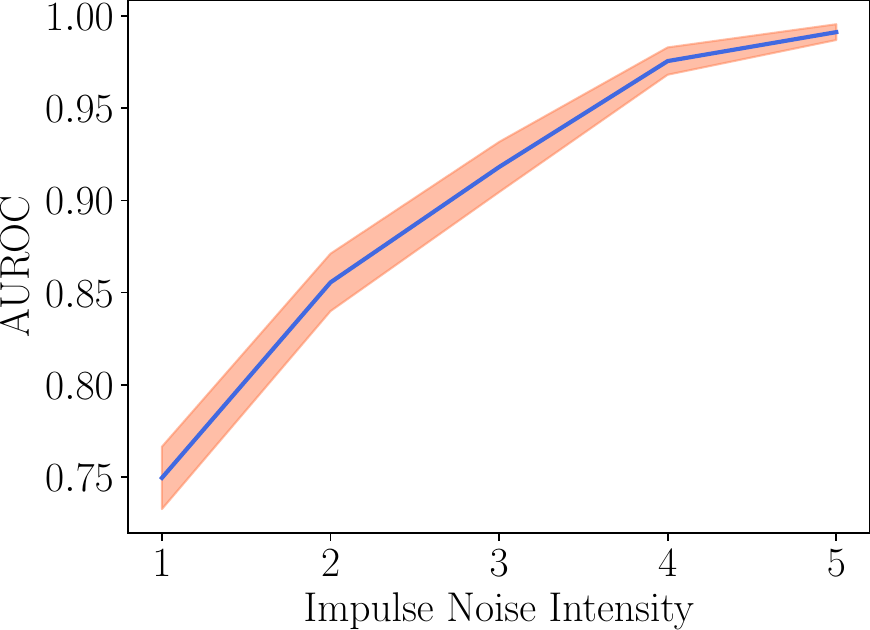}
  \caption{}
\end{subfigure}
%%%%%%%%%%%%%%%%%%%%%
\begin{subfigure}{0.333\textwidth}
  \centering  \includegraphics[width=.9\linewidth]{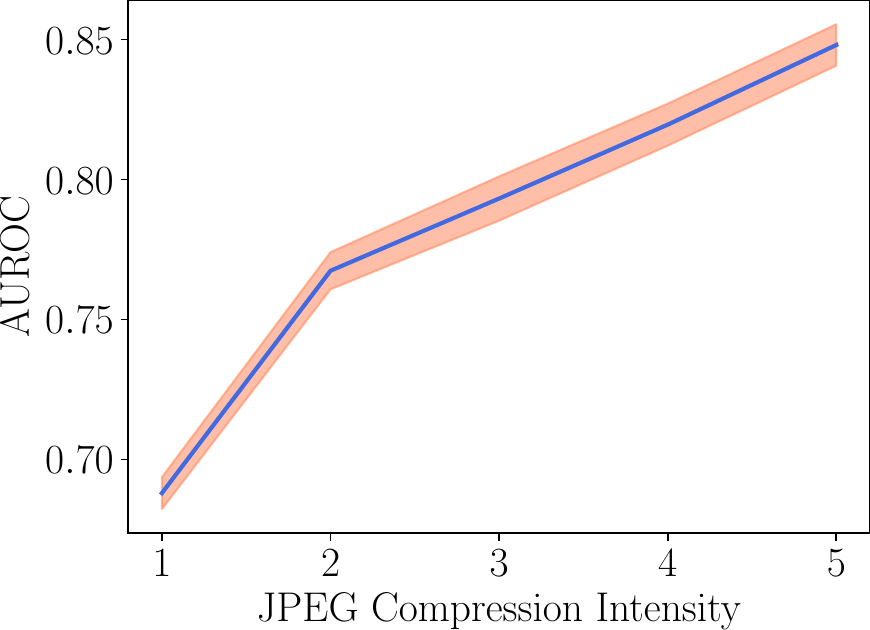}
  \caption{}
\end{subfigure}%
\begin{subfigure}{0.333\textwidth}
  \centering
  \includegraphics[width=.9\linewidth]{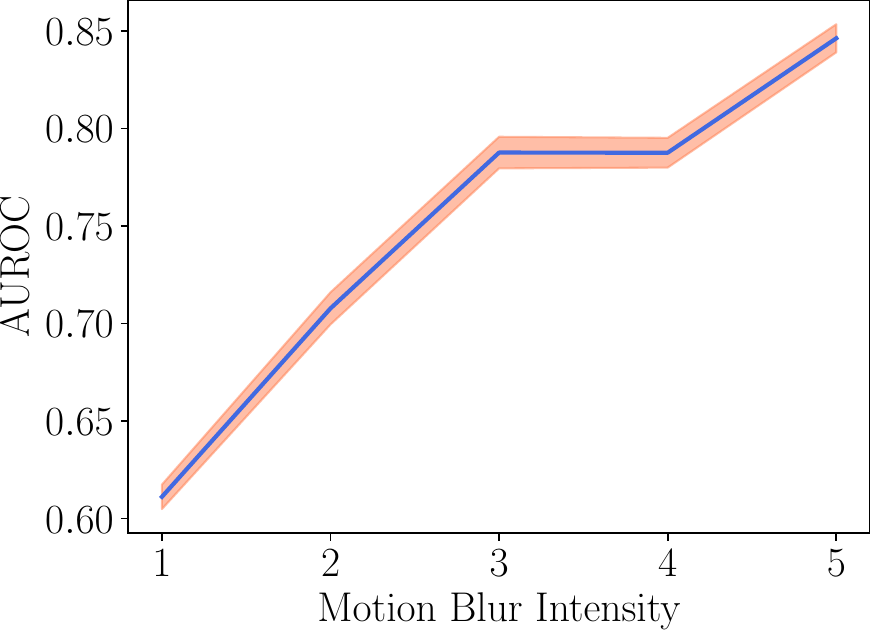}
  \caption{}
\end{subfigure}%\par\medskip
\begin{subfigure}{0.333\textwidth}
  \centering
  \includegraphics[width=.9\linewidth]{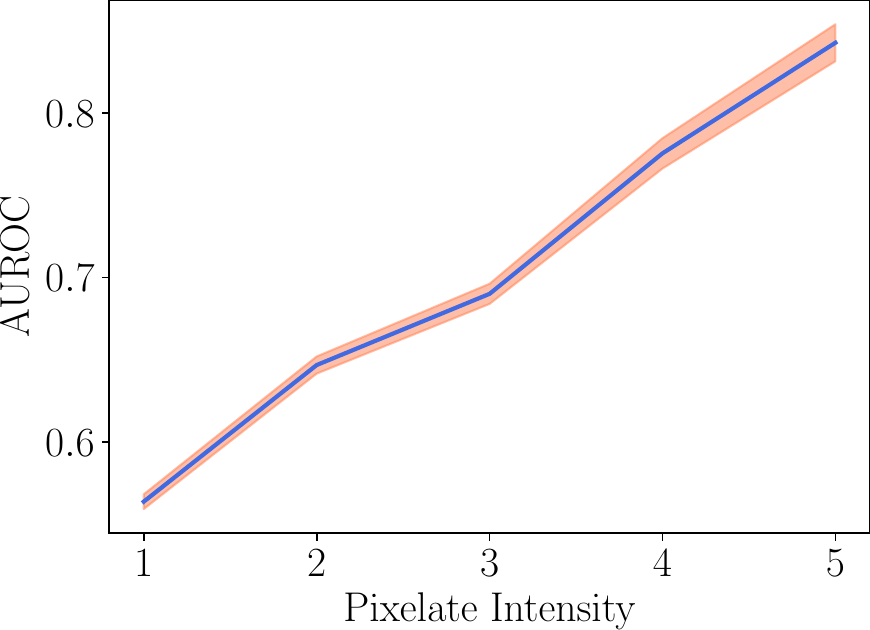}
  \caption{}
\end{subfigure}
%%%%%%%%%%%%%%%%%%%%%
\begin{subfigure}{0.333\textwidth}
  \centering  \includegraphics[width=.9\linewidth]{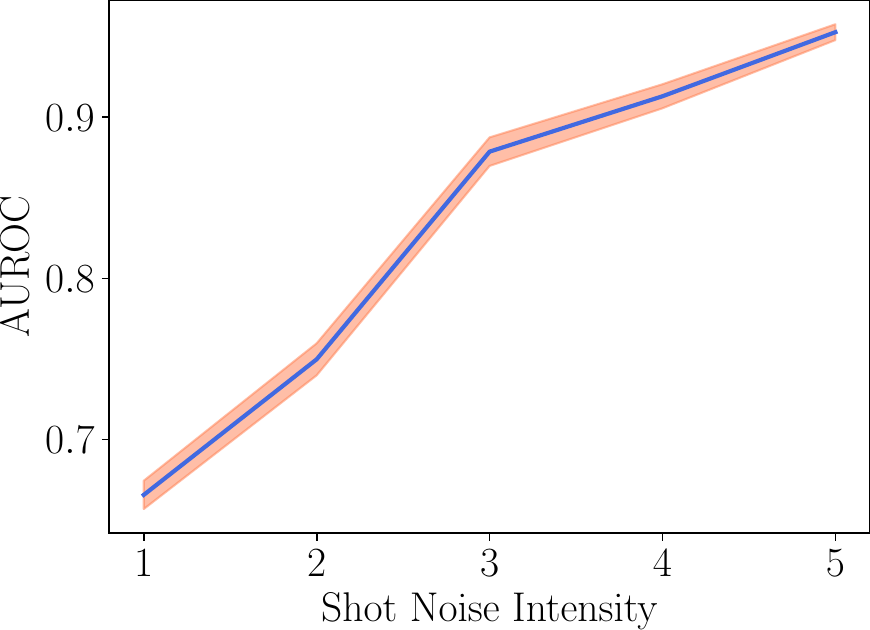}
  \caption{}
\end{subfigure}%
\begin{subfigure}{0.333\textwidth}
  \centering
  \includegraphics[width=.9\linewidth]{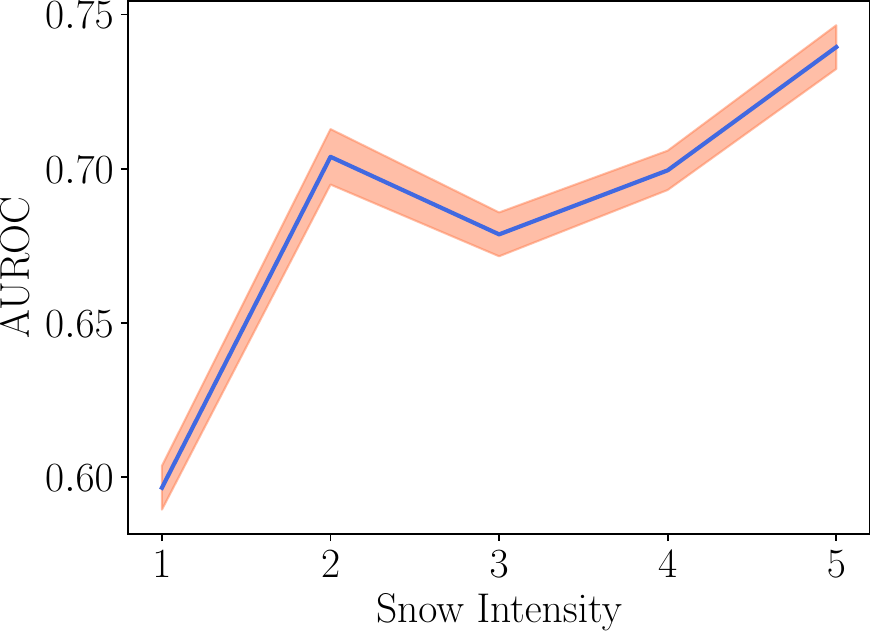}
  \caption{}
\end{subfigure}%\par\medskip
\begin{subfigure}{0.333\textwidth}
  \centering
  \includegraphics[width=.9\linewidth]{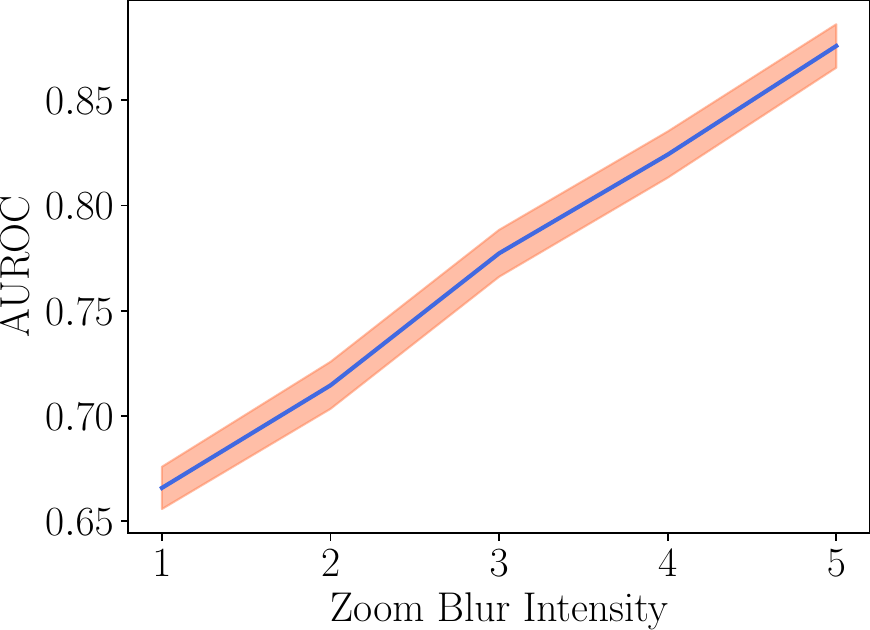}
  \caption{}
\end{subfigure}
\caption{DAB's AUROC vs corruption intensity for common corruptions to test CIFAR. The shaded area corresponds to $+/-$ one standard deviation across 10 random seeds.}
\label{fig:dab_on_corrupted}
\end{figure}
%%%%%%%%%%%%%%%%%%%%%%%%%%%%

%%%%%%%%%%%
\subsection{Qualitative Evaluation of DAB on CIFAR-10 (Section~\ref{sec:ood_experiments} 
 continued.)}\label{sec:visualize_clusters}
%%%%%%%
In \Figref{fig:5_clusters} and \Figref{fig:10_clusters_gd}, we also investigate qualitatively the rest of the IB methods examined in Sections~\ref{sec:ood_experiments} and~\ref{sec:ablation} (Tables~\ref{tab:ood_baselines},~\ref{tab:ood_vary_k}).
%%%%%%%%%%%%%
\begin{figure}[!ht]
\centering
\caption{\textbf{Qualitative evaluation with 5 entries.} We visualize the number of test data points per class assigned to each centroid at the end of three (first, middle, last) iterations of our alternating minimization algorithm (Algorithm~\ref{alg:training_ig_vib}). We notice that semantically similar classes are assigned to the same code. For example, dogs (class 5) and cats (class 3) are both represented by centroid 3. Similar observations hold for the pair of cars (class 1)/ trucks (class 9) and airplane (class 0)/ ships (class 8).}
\label{fig:5_clusters}
\includegraphics[width=0.9\textwidth]{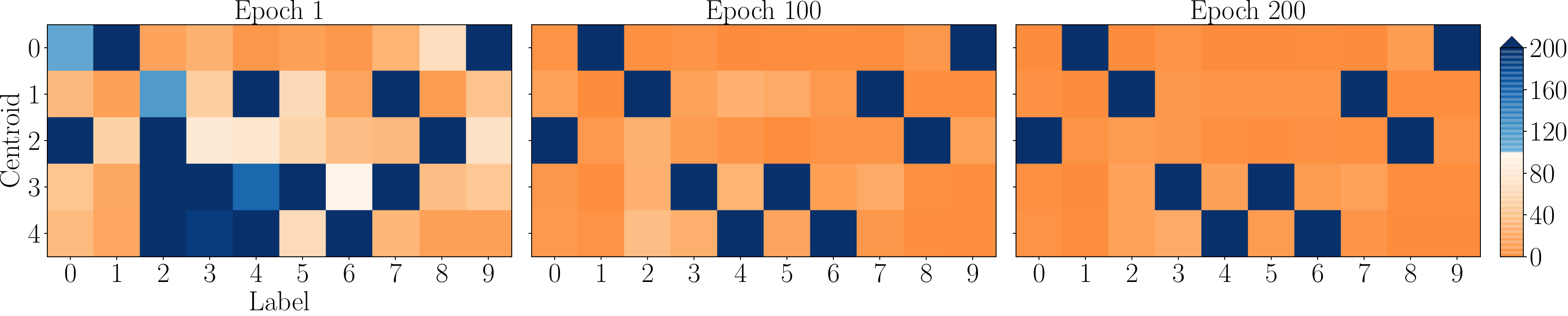}
\end{figure}
%%%%%%%%%%%%%
\begin{figure}[!ht]
\centering
 \caption{\textbf{Qualitative evaluation of a $\mathbf{10-}$component mixture marginal in the VIB trained with gradient descent.} We visualize the number of test data points per class assigned to each component at the end of three (first, middle, last) epochs when the mixture variational marginal $q(\boldsymbol{z};\boldsymbol{\phi})$ and the rest of the network (encoder and decoder) are jointly trained via gradient descent~\citep{alemi2018uncertainty}. We notice that gradient descent conflates features of different classes. This observation can help explain the inferior performance of the IB gradient descent method on OOD tasks (Table~\ref{tab:ood_baselines}). Moreover, it justifies the need for guiding optimization through the alternating minimization steps of Algorithm~\ref{alg:training_ig_vib}.}
\label{fig:10_clusters_gd}
\includegraphics[width=0.9\textwidth]{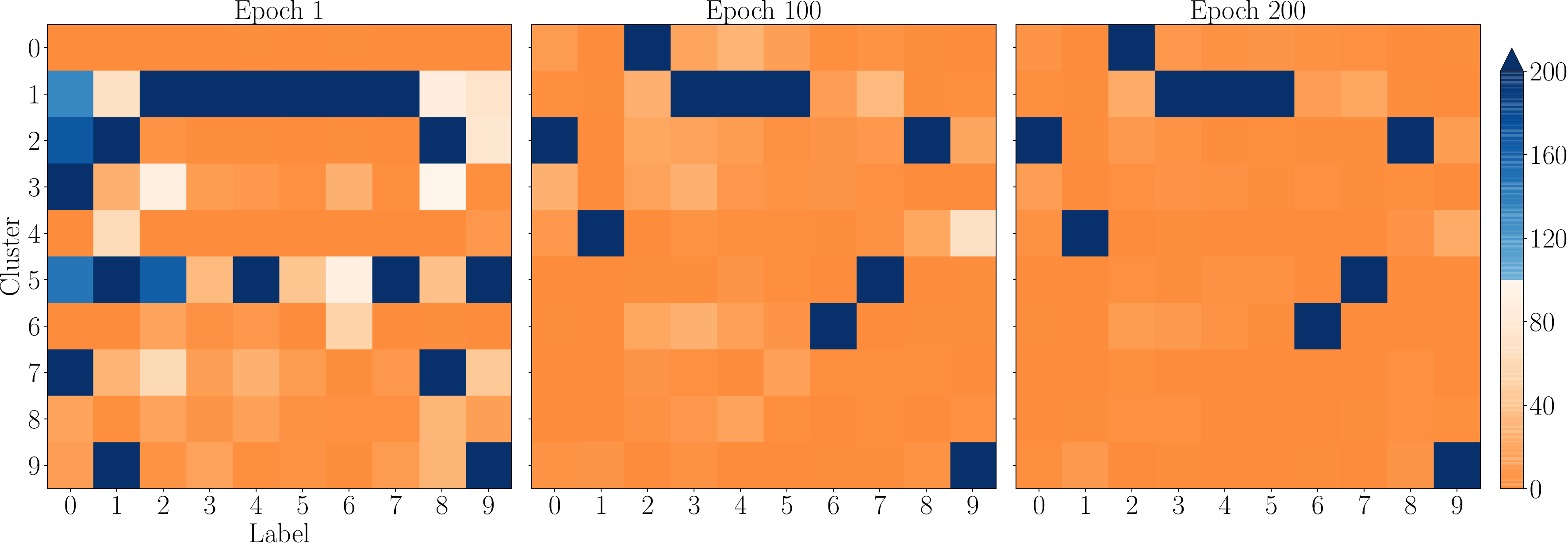}
\end{figure}
%%%%%%%%%%%%%%%%%%%%%%
Finally, \Figref{fig:calibration_qual_eval} visualizes DAB's calibration demonstrating that model's accuracy negatively correlates with its uncertainty. We verify that misclassification of a test datapoint is signaled by its large distance from the codebook.
%%%%%%%%%%%%%
\begin{figure}[!ht]
\centering
\caption{\textbf{Calibration plot of DAB on CIFAR-10 test data.} We qualitatively assess the proposed uncertainty score in terms of calibration. We train 10 models with different random seeds. For each model, we find the 20 quantiles of the estimated uncertainty on test data. We compute the accuracy for the datapoints whose uncertainty falls between two successive quantiles. We report the mean uncertainty and accuracy along with one standard deviation error bars across the runs. We see that the accuracy is higher in the quantile buckets of lower uncertainty. }\label{fig:calibration_qual_eval}
\includegraphics[width=0.5\textwidth]{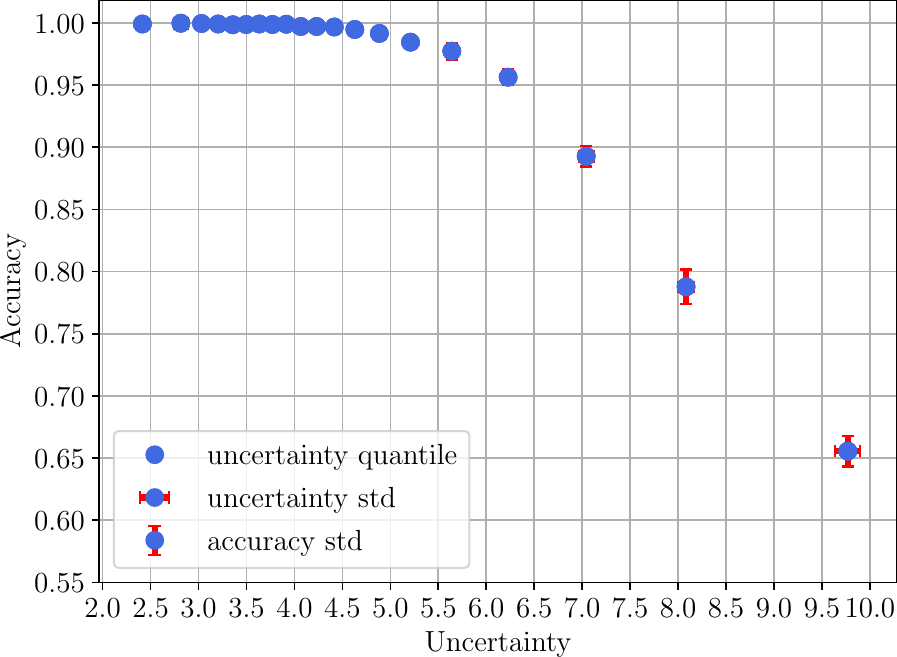}
\end{figure}
%%%%%%%%%%%%%
\subsection{Out-of-Distribution Detection on UCI Regression Tasks}\label{sec:regression_ood}
%%%%%%%%%%%%%
Currently, the bulk of uncertainty-aware methods is designed for and applied to image classification in supervised learning settings. However, as shown by~\citet{jaeger2023a}, a wide range of tasks and datasets should be considered when evaluating OOD methods. Moreover, there is an ongoing importance to effective uncertainty estimation for regression tasks, especially in unsupervised learning scenarios. For example, in deep reinforcement learning, uncertainty quantification for the estimated Q-values can be leveraged for efficient exploration~\citep{lee2021sunrise}. As already pointed out, DAB provides a unified notion of uncertainty for both regression and classification tasks.

In Table~\ref{tab:energy_efficiency}, we test the OOD capability of our model when trained on the normalized UCI, Energy Efficiency dataset~\citep{uci_datasets}. As in the image classification tasks, the positive label corresponds to the OOD inputs. The results were averaged across 10 runs. We contrast our model with ensemble methods. We see that DAB consistently demonstrates OOD capability and outperforms 4-member ensembles on all OOD tasks (of varying difficulty). In Section~\ref{sec:uci_regression_setup}, we provide the experimental details. Here, we comment that all centroids were assigned a roughly equal number of datapoints indicating the need for codebook sizes larger than one (recall that DAB with a unit-size codebook corresponds to the standard VIB~\citep{alemi2016deep}). 
%%%%%%%%%%%%%
\begin{table}[!htb]
    \centering    
     \caption{\textbf{DAB's OOD performance on the UCI energy efficiency dataset.} DAB can consistently and competitively solve a diversity of OOD regression tasks.}
     \label{tab:energy_efficiency}
    \begin{tabular}{l*{6}{c}}
    \toprule
     \textbf{OOD Dataset} &  \textbf{Model} & \multicolumn{2}{c}{ \textbf{OOD Scores}} \\
    \cmidrule(lr){3-4}
    & &  \textbf{AUROC $\uparrow$}   &  \textbf{AUPRC $\uparrow$}  \\
    \midrule
    kin8nm & DAB & $\boldsymbol{0.982 \pm 0.008}$ & $\boldsymbol{0.998 \pm 0.001}$ \\
    & Ensemble of 2 & $0.916 \pm 0.025$ & $0.992 \pm 0.003$  \\
    & Ensemble of 4 & $0.977 \pm 0.008$ & \boldsymbol{$0.998 \pm 0.001$}  \\

    \midrule

    concrete strength & DAB & \boldsymbol{$0.978 \pm  0.011$} & \boldsymbol{$0.988 \pm 0.006$} \\
    & Ensemble of 2 & $0.898 \pm 0.043$ & $0.941 \pm 0.028$  \\
    & Ensemble of 4 & $ 0.967 \pm 0.02\phantom{0}$ & $0.979 \pm 0.013$  \\

    \midrule

    protein structure & DAB & \boldsymbol{$0.989 \pm 0.017$} & $\boldsymbol{1.0}\phantom{00} \boldsymbol{\pm 0.001}$ \\
    & Ensemble of 2 & $0.875 \pm 0.059$ & $0.998 \pm 0.001$  \\
    & Ensemble of 4 & $0.971 \pm 0.018$ & $0.999 \pm 0.001$  \\

    \midrule

    boston housing & DAB & $\boldsymbol{0.988 \pm 0.008}$ & $\boldsymbol{0.988 \pm 0.007}$ \\
    & Ensemble of 2 & $0.888 \pm 0.043$ & $0.887 \pm 0.047$  \\
    & Ensemble of 4 & $0.969 \pm 0.028$ & $0.967 \pm 0.03\phantom{0}$  \\

   \bottomrule
    \end{tabular}
\end{table}
%%%%%%%%%%%%%

%%%%%%%%%%%%%
\section{Implementation details}
\subsection{Implementation details for the synthetic regression tasks (Section~\ref{sec:regression_experiments})}\label{sec:synthetic_setup}

For the example of~\Figref{fig:regression_uniform}, we generate 20 training data points from the uniform distribution $\mathcal{U}[-4, 4]$. The test data are evenly taken in $[-5, 5]$. The targets are $y=x^3+\epsilon$, where $\epsilon \sim \mathcal{N}(0,9)$. We use a single centroid to represent the whole dataset. We verify that the model's confidence and accuracy decline as we move far away from the data. In~\Figref{fig:regression_clusters}, we stress test our model under a harder variant of the first problem. In this case, we create two clusters of training data points sampled from $\mathcal{U}[-5, -2]$ and $\mathcal{U}[2, 5]$. We use two codes.

We use a network with 3 dense layers. We apply DAB to the last one. The intermediate layers have 100 hidden units and ELU non-linearity. We perform $1500$ training iterations. We use a single encoder sample during training. The optimizer of both the main network and the centroids are set to \texttt{tf.keras.optimizers.Adam} with initial learning rates $\eta_{\boldsymbol{\theta}}=0.001$ and $\eta_{\boldsymbol{\phi}}=0.01$ respectively. 
The rest of the hyperparameters are set to the default values of \texttt{tf.keras.keras.layers.Dense}. Regarding the parametrization and initialization of encoders and centroids, we follow the setup described in Section~\ref{sec:setup_ood_experiments}. 

In Table~\ref{table:regression_hyperparameters}, we provide the hyperparameters related to DAB. Note that the dataset for these tasks consists of only $20$ datapoints. Therefore, we can use the whole dataset at each gradient update step. In this case, there is no need to maintain moving averages for the update of the assignment probabilities and covariance matrices.

\begin{table}[H]
\centering
\caption{\textbf{A summary of DAB hyperparameters for the synthetic regression tasks.}}
\begin{tabular}{||c c c||} 
 \hline
 Hyperparameter & Description &  Value  \\ [0.5ex] 
 \hline\hline
 $\beta$ & {\begin{tabular}[c]{@{}c@{}}  \begin{tabular}{@{}c@{}} Regularization coefficient (\Eqref{eq:ua_ib})  \end{tabular}   \end{tabular}}  & 1.0  \\ 
 $\alpha$ & {\begin{tabular}[c]{@{}c@{}}  \begin{tabular}{@{}c@{}} Temperature (\Eqref{eq:vib_ii_e_step_sol})  \end{tabular}   \end{tabular}} & 5.0  \\
 $\mathrm{dim}(\boldsymbol{z})$ & Dimension of latent features & $8$  \\
 $k$ & Number of centroids & \begin{tabular}{@{}c@{}} 1 for \Figref{fig:regression_uniform} \\ 2 for \Figref{fig:regression_clusters} \end{tabular}    \\
 $\gamma$ & Momentum of moving averages (\Eqref{eq:momentum}) & 0.0  \\ [1ex] 
 \hline
\end{tabular}
\label{table:regression_hyperparameters}
\end{table}
\subsection{Implementation details for the CIFAR-10 experiments (Tables~\ref{tab:ood_baselines},~\ref{table:calibration_score_},~\ref{table:OoD_params_accuracy}, Sections~\ref{sec:visualize_clusters},~\ref{sec:ablation})}  \label{sec:setup_ood_experiments}
All models are trained on four 32GB V100 GPUs. The per-core batch size is 64. For fair comparisons, we train all models for 200 epochs. This number may deviate from the suggested setup of some baselines such as RegMixup ~\citep{pinto2022using} or DDU~\citep{Mukhoti_2023_CVPR} which are originally trained for 350 epochs. For the IB and GP methods, we backpropagate through a single sample. 

DAB is interleaved between the Wide ResNet 28-10 features (right after the flattened average pooling layer) and the last dense layer of the classifier. In this experiment, we use a full-covariance multivariate Gaussian for the encoder and the centroids. The encoder's network first learns a matrix $\boldsymbol{S}$ as $\boldsymbol{S}=\boldsymbol{U} \sqrt{\boldsymbol{\Lambda}}$. $\boldsymbol{U}$ is a unitary matrix. $\boldsymbol{\Lambda}$ is a positive definite, diagonal matrix. $\boldsymbol{U}$ and $\boldsymbol{\Lambda}$ are computed from the SVD decomposition of a symmetric matrix. To enforce positive definiteness of $\boldsymbol{\Lambda}$ with small initial values, we transform its entries by \texttt{softplus}($\lambda-5.0$). A similar transformation was used by~\citet{alemi2016deep}. Finally, the covariance matrix is given by: $\boldsymbol{\Sigma}=\boldsymbol{S}\boldsymbol{S}^T$.

We train the means of the centroids using \texttt{tf.keras.optimizers.Adam(learning\_rate=1e-1)}.  Only for the case $k=1$ in Table~\ref{tab:ood_vary_k}, we used \texttt{tf.keras.optimizers.Adam(learning\_rate=1e-3)}.% For the mixture marginal VIB of Table~\ref{tab:comparison_ood_ib} that is trained with gradient descent, we scale down the gradient update for the centroids' parameters by $0.01$. 
The centroid means are initialized with \texttt{tf.random\_normal\_initializer(mean=0.0,stddev=0.1)}. For the hyperparameters that are not related to the DAB, we preserve the default values used in: \\ {\small \url{https://github.com/google/uncertainty-baselines/blob/main/baselines/cifar/deterministic.py}}.

\begin{table}[H]
\centering
\caption{\textbf{A summary of DAB hyperparameters for the CIFAR-10 classification tasks.}}
\begin{tabular}{||c c c||} 
 \hline
 Hyperparameter & Description &  Value  \\ [0.5ex] 
 \hline\hline
 $\beta$ & {\begin{tabular}[c]{@{}c@{}}  \begin{tabular}{@{}c@{}} Regularization coefficient (\Eqref{eq:ua_ib})  \end{tabular}   \end{tabular}}  & 0.001  \\ 
 $\alpha$ & {\begin{tabular}[c]{@{}c@{}}  \begin{tabular}{@{}c@{}} Temperature (\Eqref{eq:vib_ii_e_step_sol})  \end{tabular}   \end{tabular}} & 1.0  \\
$\mathrm{dim}(\boldsymbol{z})$ & Dimension of latent features & 8  \\
 $k$ & Number of centroids & 10 \\
 $\gamma$ & Momentum of moving averages (\Eqref{eq:momentum}) & 0.99  \\ [1ex] 
 \hline
\end{tabular}
\label{table:ood_hyperparameters}
\end{table}
%\clearpage

\subsection{Implementation details for the ImageNet-1K experiments (Table~\ref{table:imagenet})}\label{sec:imagenet_setup}

All models are trained on four 48GB  RTX A6000 GPUs. The per-core batch size is 256. We initialize the network with the weights of a pre-trained ResNet-50 network\footnote{ \texttt{tf.keras.applications.resnet50.ResNet50}}. We train DAB only for 70 epochs.

The ResNet-50 features (without including the fully-connected layer at the top of the network) are first passed through three fully connected layers, each with 2048 units and ReLU activation. We add a residual connection between the first and last dense layer before DAB's input. In this experiment and due to the higher dimension of the latent features, we use diagonal multivariate Gaussians for the encoder and the centroids. The encoder's scale matrix is given by $\boldsymbol{S}=\text{\texttt{diag}}(\text{\texttt{softplus}}(\boldsymbol{o}-5.0))$, where $\boldsymbol{o}$ are the encoder's outputs corresponding to the covariance. The covariance matrix is given by: $\boldsymbol{\Sigma}=\boldsymbol{S}\boldsymbol{S}^T$. Finally, we use Eq. 9 of \citet{davis2006differential} to update the codebook's covariance matrices where only the diagonal entries are computed.

To improve model's calibration, we add a max margin-loss term in the objective function of~\Eqref{eq:ua_ib} for the misclassified datapoints:
\begin{align}
\ell(\boldsymbol{x}) = \max(0,\text{U}_{lb}-\mathrm{uncertainty}(\boldsymbol{x}))
\label{eq:u_lb}.
\end{align}
$\mathrm{uncertainty}(\boldsymbol{x})$ is defined in~\Eqref{eq:distance_3}. This term encourages higher model's uncertainty for the mispredicted training datapoints. In the experiment, we set the uncertainty lower bound as $\text{U}_{lb}=100$. Moreover, only the correctly classified training examples are quantized by the codebook~(\Eqref{eq:r_d_f_c_dtrain}). For the OOD experiments, we quantize all training datapoints regardless the classification outcome. Therefore, the loss term in~\Eqref{eq:u_lb} is omitted. 

Table~\ref{table:ood_hyperparameters_imagenet} provides DAB's hyperparameters when we backpropagate to ResNet-50. Table~\ref{table:ood_hyperparameters_imagenet_pretrained} provides DAB's hyperparameters when the ResNet-50 weights are frozen during training. In the first case, the main network is trained using \texttt{tf.keras.optimizers.SGD(learning\_rate=1e-1)}. In the second case, encoder's dense layers and the decoder are trained using \texttt{tf.keras.optimizers.SGD(learning\_rate=5e-2)}. We train centroids' means using \texttt{tf.keras.optimizers.Adam}. The centroid means are initialized with \texttt{tf.random\_normal\_initializer(mean=0.0,stddev=0.1)}. %Finally, we note that DAB's regularization coefficient is given by $\lambda=\alpha \times \beta$~(\Eqref{eq:ua_ib}). 
For the rest of the hyperparameters, we preserve the default values used in: \\ {\small \url{https://github.com/google/uncertainty-baselines/blob/main/baselines/imagenet/deterministic.py}}.

\begin{table}[H]
\centering
\caption{\textbf{Hyperparameters for DAB with ResNet-50 fine-tuning for the ImageNet-1K classification tasks  (Table~\ref{table:imagenet}). }}
\begin{tabular}{||c c c||} 
 \hline
 Hyperparameter & Description &  Value  \\ [0.5ex] 
 \hline\hline
  $\eta_{\boldsymbol{\phi}}$ & {\begin{tabular}[c]{@{}c@{}}  \begin{tabular}{@{}c@{}} Codebook's learning rate   \end{tabular}   \end{tabular}}  & \begin{tabular}{@{}c@{}} 0.4 for OOD  \\ 0.1 for Calibration \end{tabular} \\
 $\beta$ & {\begin{tabular}[c]{@{}c@{}}  \begin{tabular}{@{}c@{}} Regularization coefficient (\Eqref{eq:ua_ib})  \end{tabular}   \end{tabular}}  & \begin{tabular}{@{}c@{}} 0.005 for OOD  \\ 0.01 for Calibration \end{tabular}  \\ 
 % &  & \phantom{0}\phantom{000}0.010 (for Calibration)  \\ 
 $\alpha$ & {\begin{tabular}[c]{@{}c@{}}  \begin{tabular}{@{}c@{}} Temperature (\Eqref{eq:vib_ii_e_step_sol})  \end{tabular}   \end{tabular}} & 2.0  \\
$\mathrm{dim}(\boldsymbol{z})$ & Dimension of latent features & 80  \\
 $k$ & Number of centroids & 1000 \\
 $\gamma$ & Momentum of moving averages (\Eqref{eq:momentum}) & 0.99  \\ [1ex] 
 \hline
\end{tabular}
\label{table:ood_hyperparameters_imagenet}
\end{table}
\begin{table}[H]
\centering
\caption{\textbf{Hyperparameters for DAB without ResNet-50 fine-tuning for the ImageNet-1K classification tasks (Table~\ref{table:imagenet}). }}
\begin{tabular}{||c c c||} 
 \hline
 Hyperparameter & Description &  Value  \\ [0.5ex] 
 \hline\hline
  $\eta_{\boldsymbol{\phi}}$ & {\begin{tabular}[c]{@{}c@{}}  \begin{tabular}{@{}c@{}} Codebook's learning rate   \end{tabular}   \end{tabular}}  & \begin{tabular}{@{}c@{}} 0.4 for OOD  \\ 0.1 for Calibration \end{tabular} \\
 $\beta$ & {\begin{tabular}[c]{@{}c@{}}  \begin{tabular}{@{}c@{}} Regularization coefficient (\Eqref{eq:ua_ib})  \end{tabular}   \end{tabular}}  & \begin{tabular}{@{}c@{}} 0.0025 for OOD  \\ 0.02 for Calibration \end{tabular}  \\ 
 % &  & \phantom{0}\phantom{000}0.010 (for Calibration)  \\ 
 $\alpha$ & {\begin{tabular}[c]{@{}c@{}}  \begin{tabular}{@{}c@{}} Temperature (\Eqref{eq:vib_ii_e_step_sol})  \end{tabular}   \end{tabular}} & 2.0  \\
$\mathrm{dim}(\boldsymbol{z})$ & Dimension of latent features & 80  \\
 $k$ & Number of centroids & 1000 \\
 $\gamma$ & Momentum of moving averages (\Eqref{eq:momentum}) & 0.99  \\ [1ex] 
 \hline
\end{tabular}
\label{table:ood_hyperparameters_imagenet_pretrained}
\end{table}
%\clearpage
\subsection{Implementation details for the UCI regression experiments (Section~\ref{sec:regression_ood})}\label{sec:uci_regression_setup}

The optimizer of the codebook was set to \texttt{tf.keras.optimizers.Adam(learning\_rate=1e-1)}. The architecture consists of an MLP network with one hidden layer of dimension 50 and ReLU nonlinearity. The exact backbone architecture along with the hyperparameters that are not related to the DAB were kept the same and can be found here:
\url{https://github.com/google/uncertainty-baselines/tree/main/baselines/uci}. 

As for the rest of experiments, we applied DAB between the penultimate and output layer of the architecture. Regarding the parametrization and initialization of encoders and centroids, we follow the setup described in Section~\ref{sec:setup_ood_experiments}. Table~\ref{table:ood_uci_hyperparameters} reports the DAB hyperparameters used for this task.

\begin{table}[H]
\centering
\caption{\textbf{A summary of DAB hyperparameters for the UCI regression tasks.}}
\begin{tabular}{||c c c||} 
 \hline
 Hyperparameter & Description &  Value  \\ [0.5ex] 
 \hline\hline
 $\beta$ & {\begin{tabular}[c]{@{}c@{}}  \begin{tabular}{@{}c@{}} Regularization coefficient (\Eqref{eq:ua_ib})  \end{tabular}   \end{tabular}}  & 0.001  \\ 
 $\alpha$ & {\begin{tabular}[c]{@{}c@{}}  \begin{tabular}{@{}c@{}} Temperature (\Eqref{eq:vib_ii_e_step_sol})  \end{tabular}   \end{tabular}} & 1.0  \\
$\mathrm{dim}(\boldsymbol{z})$ & Dimension of latent features & 4  \\
 $k$ & Number of centroids & 2 \\
 $\gamma$ & Momentum of moving averages (\Eqref{eq:momentum}) & 0.99  \\ [1ex] 
 \hline
\end{tabular}
\label{table:ood_uci_hyperparameters}
\end{table}